\documentclass{llncs}
\usepackage{amsmath,amssymb}
\usepackage[table]{xcolor}
\usepackage{caption}
\usepackage{subcaption}
\usepackage[]{mathtools}
\usepackage{bbm}
\usepackage{mdframed}
\usepackage[sort,compress]{cite}
\usepackage[numbers,sectionbib]{natbib}
\usepackage[many]{tcolorbox}
\usepackage{algorithmicx}
\usepackage{algorithm}
\usepackage[]{algpseudocode}
\algnewcommand\algorithmicinput{\textbf{Input:}}
\algnewcommand\INPUT{\item[\algorithmicinput]}
\usepackage{moreverb,url}
\usepackage[bookmarksopen,bookmarksnumbered,colorlinks=true, linkcolor=cyan,
citecolor = blue, urlcolor = cyan,filecolor=black , pagebackref=false,hypertexnames=false, plainpages=false, pdfpagelabels ]{hyperref}
\usepackage{pgfplots}
\usepgfplotslibrary{fillbetween}
\usepackage[perpage]{footmisc}
\usepackage{soul}
\usepackage{wrapfig}

\usepackage{tikz} 
\usepackage{tkz-graph}
\usepackage{tkz-berge}
	\usetikzlibrary{backgrounds,fit,shapes,snakes,arrows,shapes.geometric,positioning}
	\usetikzlibrary{intersections,patterns,shapes.misc}
	\usetikzlibrary{decorations.pathmorphing}
	\tikzstyle{block} = [rectangle, rounded corners, minimum width=3cm, minimum height=1cm,text centered, draw=black, fill=red!30]
	\tikzstyle{new} = [rectangle, rounded corners, minimum width=1cm, minimum
	height=1cm,text centered, draw=black, fill=blue!10!white, dashed]
	\tikzstyle{arrow} = [thick,->,>=stealth]
	\usetikzlibrary{calc, quotes}
\usepackage{fontawesome}
\DeclareFontFamily{OT1}{pzc}{}
\DeclareFontShape{OT1}{pzc}{m}{it}{<-> s * [1.200] pzcmi7t}{}
\DeclareMathAlphabet{\mathpzc}{OT1}{pzc}{m}{it}

\AtBeginDocument{%
  \DeclareMathAlphabet\PazoBB{U}{fplmbb}{m}{n}%
}

\usepackage{dsfont}
\usepackage{microtype}

\usepackage{diagbox}
\usepackage{multirow}
\usepackage{dashbox}

\usepackage{sidecap}
\algrenewcommand\textproc{}%


\let\oldthempfootnote\thempfootnote
\def\thempfootnote{\text{\oldthempfootnote}}
\usepackage{enumitem}

\setlist[itemize]{leftmargin=*}
\setlist[enumerate]{leftmargin=*}
\colorlet{lightgray}{green!4}

\newcommand{\GG}{\mathcal{G}}

\newcommand{\VV}{\mathcal{V}}

\newcommand{\EE}{\mathcal{E}}

\newcommand{\bell}{\boldsymbol{\ell} }

\newcommand{\ppp}{\boldsymbol\pi}
\newcommand{\AAA}{\mathbf{A}}

\newcommand{\Acal}{\mathcal{A}}

\newcommand{\Scal}{\mathcal{S}}
\newcommand{\Ncal}{\mathcal{N}}
\newcommand{\Bcal}{\mathcal{B}}

\newcommand{\Wcal}{\mathcal{W}}

\newcommand{\Fcal}{\mathcal{F}}

\DeclareMathAlphabet\mathbfcal{OMS}{cmsy}{b}{n}

\newcommand{\Rset}{\mathbb{R}}

\DeclareMathOperator*{\argmax}{arg\,max}

\newcommand{\Ecal}{\mathcal{E}}
\newcommand{\Vcal}{\mathcal{V}}
\newcommand{\Ucal}{\mathcal{U}}

\newcommand{\kcomp}{k}
\newcommand{\kcomm}{b}

\let\oldsim\sim
\renewcommand{\sim}{\overset{e}{\oldsim}}

\newcommand{\fe}{f}
\newcommand{\fv}{g}
\newcommand{\edg}{\mathsf{edges}}
\newcommand{\OPTe}{\text{\text{OPT}$_1$}}
\newcommand{\OPTv}{\text{\text{OPT}$_2$}}

\newcommand{\Vgrd}{\VV_\text{grd}}
\newcommand{\Egrd}{\Ecal_\text{grd}}

\definecolor{kkGreen}{RGB}{201,232,206}
\definecolor{kkRed}{RGB}{255,196,215}
\definecolor{kkBlue}{RGB}{214,226,255}
\newcommand{\Gall}{\GG_\mathsf{x}}
\newcommand{\Vall}{\VV_\mathsf{x}}
\newcommand{\Eall}{\EE_\mathsf{x}}

\newcommand{\Qcal}{\mathcal{Q}}
\newcommand{\TotalUni}{\textbf{TU}}
\newcommand{\TotalNon}{\textbf{TN}}
\newcommand{\IndivUni}{\textbf{IU}}
\newcommand{\Gcal}{\Gall}
\newcommand{\Cover}{\mathsf{cover}}

\makeatletter
\newcommand{\@chapapp}{\relax}%
\makeatother
\newcommand*\samethanks[1][\value{footnote}]{\footnotemark[#1]}

\usepackage[title,header]{appendix}
\begin{document}

\mainmatter              % start of the contributions
\title{Resource-Aware Algorithms for\\Distributed Loop Closure
Detection\\with Provable Performance Guarantees}
%\title{Computation and Communication Codesign for Distributed Data Association}
%\title{Sparse D-Optimal SLAM Through Graph Synthesis}
%
\titlerunning{}  % abbreviated title (for running head)
\author{Yulun Tian\thanks{Equal contribution.} 
\and Kasra Khosoussi\samethanks
\and Jonathan P.~How} 
\authorrunning{Tian et al.} % abbreviated author list (for running head)

\institute{Laboratory for Information and Decision Systems\\Massachusetts
	Institute of Technology\\Cambridge, MA, USA.\\
	\email{\{yulun,kasra,jhow\}@mit.edu}}

\maketitle              % typeset the title of the contribution

%\begin{enumerate}
%  \item \Kcomment{look at ORBSLAM or fabmap or Scaramuzza's paper or  other
%	papers to find data about the complexity/cost of geometric verification}
%  \item \Kcomment{check notation consistency: V/S for vertex set, A/E for edge set, b/k for budgets}
%  \item \Kcomment{Figure 1: black edge}
%  \item \Kcomment{Yulun's idea on document matching and BoW - see Scaramuzza paper}
%  \item \Ycomment{Is codesign a good name?}
%  \item \Ycomment{runtime + lazy greedy}
%  \item \Ycomment{contribution as a list?}
%  \item \Ycomment{Fix algorithm names. Abbreviation?}
%  \item \Ycomment{Quantify $\kcomp$ and $\kcomm$, e.g., MB}
%  \item \Kcomment{Check Remark 1}
%  \item \Ycomment{Minor: change link color}
%  \item \Ycomment{Talk more generally in the first paragraph of intro?}
%\end{enumerate}
%\vspace{-0.9cm}
\begin{abstract}
Inter-robot loop closure detection, e.g., for collaborative
simultaneous localization and mapping (CSLAM), is a fundamental capability for
many multirobot applications in GPS-denied regimes. In real-world scenarios,
this is a resource-intensive process that involves exchanging observations and
verifying potential matches. This poses severe challenges
especially for small-size and low-cost robots with various operational and
resource constraints that limit, e.g., energy consumption, communication
bandwidth, and computation capacity. This paper presents resource-aware
algorithms for distributed inter-robot loop closure detection. In particular, we
seek to select a subset of potential inter-robot loop closures that maximizes a
monotone submodular performance metric without exceeding computation and
communication budgets. We demonstrate that this problem is in general NP-hard,
and present efficient approximation algorithms with provable performance
guarantees.
A convex relaxation scheme is used to certify near-optimal performance of the
proposed framework in real and synthetic SLAM benchmarks.

%	Consider $r$ rendezvousing agents that seek to find matches between their
%	respective observations. Robots can broadcast their observations and, upon
%	receiving others' observations, can verify potential matches. Under
%	budgeted communication and computation, robots must 
%	collectively decide \emph{who} should share \emph{what} with \emph{whom} and
%	verify \emph{which} potential matches. Such decisions are driven by a
%	task-dependent objective that assigns a value \ldots
%	\keywords{Inter-robot loop closure detection, Collaborative SLAM, Maximum Coverage}
\end{abstract}
%\vspace{-0.9cm}
\section{Introduction}
%\vspace{-0.2cm}
Multirobot systems provide efficient and sustainable solutions to many large-scale missions.
Inter-robot loop closure detection, e.g., for collaborative simultaneous localization and mapping (CSLAM), 
is a fundamental capability necessary for many such applications in GPS-denied environments.
Discovering inter-robot loop closures requires (i) exchanging observations
between rendezvousing robots, and (ii) collectively verifying potential matches.
In many real-world scenarios, this is a resource-intensive process with a large search space due to, e.g.,
perceptual ambiguity, infrequent rendezvous, and long-term missions \cite{cieslewski2017efficient,ila2010information,kasra18ijrr,Giamou18_ICRA}. 
This task becomes especially challenging for prevalent small-size and low-cost platforms that are subject to
various operational or resource constraints such as limited battery, low-bandwidth communication, and limited computation capacity.
It is thus crucial for such robots to be able to seamlessly adapt to such constraints and intelligently utilize available on-board resources.
Such flexibility also enables robots to
explore the underlying trade-off between resource consumption and performance,
which ultimately can be exploited to save mission-critical resources. 
Typical ad hoc schemes and heuristics only offer partial remedies and often suffer from arbitrarily poor worst-case performance.
This thus motivates the design of reliable \emph{resource-aware} frameworks that provide performance guarantees.

%
%\begin{wrapfigure}{r}{0.5\textwidth}
%  \vspace{-20pt}
%  \begin{center}
%		\includegraphics[width=0.4\textwidth]{figures/sgrd_apx.eps}
%  \end{center}
%  \vspace{-20pt}
%	\caption{\small \textsc{\small S-Greedy} approximation factor as a function of $\kappa \triangleq \kcomm / \kcomp$ (assuming $\lfloor \kcomp / \Delta \rfloor \approx \kcomp / \Delta$) 
%  with different maximum degree $\Delta$ of the exchange graph.}
%  \vspace{-20pt}
%\end{wrapfigure}
This paper presents such a resource-aware framework 
for distributed inter-robot loop closure detection. 
More specifically, given budgets on
computation (i.e., number of verification attempts) and communication
(i.e., total amount of data transmission), we seek to select a
budget-feasible subset of potential loop closures that maximizes a
monotone submodular performance metric. 
To verify a potential inter-robot loop closure, at least one of the corresponding two robots
must share its observation (e.g., image keypoints or laser scan) with the other robot. 
Thus, we need to address the following questions simultaneously: 
(i) which feasible subset of observations should robots share with each other?
And (ii) which feasible subset of potential loop closures should be selected for verification?
This problem generalizes previously studied NP-hard problems that only consider budgeted computation (unbounded communication) \cite{kasra16wafr,kasra18ijrr} or vice versa \cite{tian18}, and is therefore NP-hard. 
Furthermore, algorithms proposed in these works are incapable of 
accounting for the impact of their decisions on both resources,
and thus suffer from arbitrarily poor worst-case performance in real-world scenarios.
In this paper, we provide simple efficient approximation algorithms 
with provable performance guarantees under budgeted computation and communication.

\subsubsection*{Contributions.} Our algorithmic contributions for the 
resource-aware distributed loop closure detection problem are the following:
\begin{enumerate}
  \item A constant-factor approximation algorithm for maximizing the expected number of
	true loop closures subject to computation and communication budgets. The established performance guarantee carries
	over to any modular objective.
  \item Approximation algorithms for general monotone submodular performance metrics. The established performance guarantees depend on 
  	available/allocated resources as well as the extent of perceptual ambiguity
	and problem size.
\end{enumerate}
We perform extensive evaluations of the proposed algorithms using real and synthetic SLAM benchmarks, and use 
a convex relaxation scheme to certify near-optimality 
in experiments where computing optimal solution by exhaustive search is infeasible.

%\vspace{-0.4cm}
\subsubsection*{Related Work.}
%Cost-effective multi-robot systems must particularly be capable of managing
%their limited mission-critical resources available onboard.  As a result, 
Resource-efficient CSLAM in general
\cite{choudhary2017,paull2016unified}, and data-efficient
distributed inter-robot loop closure detection
\cite{Giamou18_ICRA,CieslewskiChoudhary17,cieslewski2017efficient,tian18,choudhary2017}
in particular have been active areas of research in recent years.  
\citet{cieslewski2017efficient} and \citet{CieslewskiChoudhary17}
propose effective heuristics to reduce data transmission in the so-called ``online query''
phase of search for \emph{potential} inter-robot loop closures, during which robots
exchange \emph{partial} queries and use them to search for promising potential matches. 
Efficiency of the more 
resource-intensive phase of \emph{full} query exchange (e.g., 
full image keypoints, point clouds) is addressed in \cite{Giamou18_ICRA}.
\citet{Giamou18_ICRA} study the optimal \emph{lossless} data exchange problem
between a pair of rendezvousing robots. In particular, they show that the optimal
exchange policy is closely related to the minimum weighted vertex cover of the
so-called \emph{exchange graph} (Figure~\ref{fig:gex}); this was later extended to general
$r$-rendezvous \cite{tian18}.
None of the abovementioned works, however, consider explicit budget constraints.

In \cite{tian18} we consider the \emph{budgeted} data exchange
problem ($b$-DEP) in which robots are subject to a communication budget. 
Specifically, \cite{tian18} provides provably near-optimal approximation
algorithms for maximizing a monotone submodular performance metric subject to
such communication budgets. 
%Similarly, the computational computation budgets
%have also been proposed and addressed in applications such as 
%large-scale structure-from-motion \cite{heinly2015,Raguram2012BMVC} and SLAM.
On the other hand, 
prior works on measurement selection in
SLAM \cite{kasra16wafr,kasra18ijrr,carlone2017attention} (rooted in
\cite{davison2005active,ila2010information}) under computation budgets
provide similar performance
guarantees for cases where robots are subject to a cardinality
constraint on, e.g., number of edges added to the pose graph, 
or number of verification attempts to discover loop closures \cite{kasra16wafr,kasra18ijrr}.
Note that bounding either quantity also helps to reduce the computational cost of solving the underlying inference problem. 
The abovementioned works assume unbounded communication or computation.
In real-world applications, however, robots need to seamlessly adapt to \emph{both} budgets.
%However, due to the their tightly coupled nature, 
%designing algorithms that simultaneously handle both budgets while guaranteeing good performance becomes a challenging problem.
The present work addresses this need by presenting resource-aware algorithms 
with provable performance guarantees that can operate in such regimes.

\subsubsection*{Notation and Preliminaries}
Bold lower-case and upper-case letters are reserved for vectors and matrices, respectively.
%$\AAA \succeq \zero$ (resp., $\AAA \succ \zero$) means that $\AAA$ is
%positive semidefinite (resp., positive definite).
%$\Gcal$ denotes a common SLAM pose-graph. $\Vcal$ is the set of poses in
%$\Gcal$ and their corresponding observations. $E$ and $\Lcal$ represent sets of
%potential loop closures.  Unless otherwise mentioned, $\Gcal$ and $\Vcal$
%denote a graph and its vertex set, respectively.  $\Vcover(\Lcal)$ denotes the
%minimum weighted vertex cover of the graph induced by the edge set $\Lcal$
%where vertices are weighted by a positive weight function $w$.  Finally, the
%value of $\Vcover(\Lcal)$ is denoted by $\vcover(\Lcal) \triangleq \sum_{v \in
%\Vcover(\Lcal)} w(v)$.
%\begin{definition}
%  \normalfont
%  \label{def:NMS}
A set function $f: 2^{\Wcal} \to \Rset_{\geq 0}$ for a finite $\Wcal$ is
\emph{normalized}, \emph{monotone}, and \emph{submodular} (NMS) if it satisfies the following
properties:
(i) normalized: $f(\varnothing) = 0$;
(ii) monotone: for any $\Ecal \subseteq \Bcal$, $f(\Acal) \leq f(\Bcal)$;
and (iii) submodular: $f(\Acal) + f(\Bcal) \geq f(\Acal \cup \Bcal) + f(\Acal
\cap \Bcal)$ for any $\Acal, \Bcal \subseteq \Wcal$.
In addition, $f$ is called \emph{modular} if both $f$ and $-f$ are submodular.
%\end{definition}
For any set of edges $\EE$, $\Cover(\EE)$
denotes the set of all vertex covers of $\EE$.
For any set of vertices $\VV$, $\edg(\VV)$ denotes the set of all edges incident to at least one vertex in $\VV$.
Finally, $\uplus$ denotes the union of disjoint sets, i.e., $\Acal \uplus \Bcal
= \Acal \cup \Bcal$ and implies $\Acal \cap \Bcal = \varnothing$.

%\vspace{-0.4cm}
\section{Problem Statement}
\label{sec:problemDefinition}

%\vspace{-0.4cm}

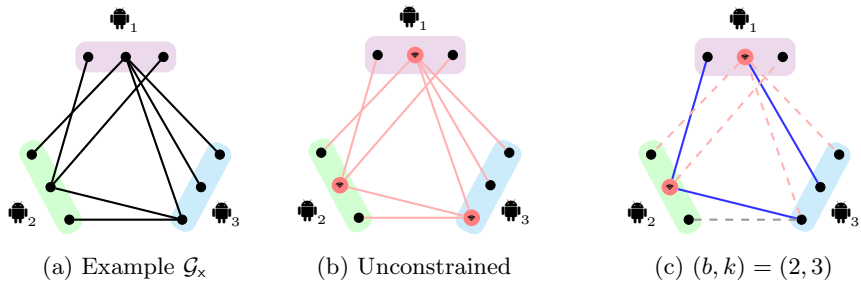
\begin{figure}[t]
	\centering
	%  \bipgraph{4}{{1,2,3,4},{1},{1},{1}}
	\hspace*{\fill}%
	\begin{subfigure}[t]{0.29\textwidth}
		\centering
		\begin{tikzpicture}[scale=1]
		\tikzstyle{vertex}=[circle,fill,scale=0.4,draw]
		\tikzstyle{special vertex}=[circle,fill=red,scale=0.4,draw]
		\tikzstyle{square vertex}=[rectangle,fill,scale=0.5,draw]
		\tikzstyle{diamond vertex}=[regular polygon,regular polygon
		sides=3,rotate=45,fill,scale=0.3,draw]
		\node[vertex] at (1,2.598076) (a1) {};
		\node[vertex] at (1.5,2.598076) (a2) {};
		\node[vertex] at (2,2.598076) (a3) {};
		\node[vertex] at (0.25,1.299037) (b1) {};
		\node[vertex] at (0.5,0.866025) (b2) {};
		\node[vertex] at (0.75, 0.433013) (b3) {};
		\node[vertex] at (2.25,0.433013) (c1) {};
		\node[vertex] at (2.5,0.866025) (c2) {};
		\node[vertex] at (2.75,1.299037) (c3) {};
		%%%%%%%
		\node (r1) at (1.5,3.098076) [] {\faAndroid$_{_1}$};
		\node (r2) at (0.1464466,0.51247) [] {\faAndroid$_{_2}$};
		\node (r3) at (2.85355, 0.51247) [] {\faAndroid$_{_3}$};
		%%%%%%%
		\draw[thick](a1) -- (b2);
		\draw[thick](a2) -- (b1);
		\draw[thick](a2) -- (c2);
		\draw[thick](a2) -- (c3);
		\draw[thick](b2) -- (c1);
		\draw[thick](c1) -- (a2);
		\draw[thick](b2) -- (a3);	  
		\draw[thick](b3) -- (c1);
		%\node [rounded corners, draw=red, rotate fit=45, fit=(b1) (b2) (b3)] {};
		\begin{pgfonlayer}{background}
			\node[fit=(a1)(a2)(a3),rounded corners,fill=violet!15,inner xsep=3pt,
	inner ysep=4pt] {};
			\node[fit=(b1)(b2)(b3),rounded corners,fill=green!18,inner xsep=-4pt,
	inner ysep=5pt,
			rotate=28] {};
			\node[fit=(c1)(c2)(c3),rounded corners,fill=cyan!18,inner xsep=-4pt,
	inner ysep=4pt,
			rotate=152] {};
		\end{pgfonlayer}
		\end{tikzpicture}
		\caption{Example $\Gcal$}
		\label{fig:gex}
	  \end{subfigure}\hspace{0.2cm}
	\hfill
	\begin{subfigure}[t]{0.29\textwidth}
		\centering
		\begin{tikzpicture}[scale=1]
		\tikzstyle{vertex}=[circle,fill,scale=0.4,draw]
		\tikzstyle{special vertex}=[circle,fill=red!50,scale=0.3]
		\tikzstyle{square vertex}=[rectangle,fill,scale=0.5,draw]
		\tikzstyle{diamond vertex}=[regular polygon,regular polygon
		sides=3,rotate=45,fill,scale=0.3,draw]
		\node[vertex] at (1,2.598076) (a1) {};
		\node[special vertex] at (1.5,2.598076) (a2) {\faWifi};
		\node[vertex] at (2,2.598076) (a3) {};
		\node[vertex] at (0.25,1.299037) (b1) {};
		\node[special vertex] at (0.5,0.866025) (b2) {\faWifi};
		\node[vertex] at (0.75, 0.433013) (b3) {};
		\node[special vertex] at (2.25,0.433013) (c1) {\faWifi};
		\node[vertex] at (2.5,0.866025) (c2) {};
		\node[vertex] at (2.75,1.299037) (c3) {};
		%%%%%%%
		\node (r1) at (1.5,3.098076) [] {\faAndroid$_{_1}$};
		\node (r2) at (0.1464466,0.51247) [] {\faAndroid$_{_2}$};
		\node (r3) at (2.85355, 0.51247) [] {\faAndroid$_{_3}$};
		%%%%%%%
		\draw[thick,red!30](a1) -- (b2);
		\draw[thick,red!30](a2) -- (b1);
		\draw[thick,red!30](a2) -- (c2);
		\draw[thick,red!30](a2) -- (c3);
		\draw[thick,red!30](b2) -- (c1);
		\draw[thick,red!30](c1) -- (a2);
		\draw[thick,red!30](b2) -- (a3);	  
		\draw[thick,red!30](b3) -- (c1);
		%\node [rounded corners, draw=red, rotate fit=45, fit=(b1) (b2) (b3)] {};
		\begin{pgfonlayer}{background}
			\node[fit=(a1)(a2)(a3),rounded corners,fill=violet!15,inner xsep=3pt,
	inner ysep=4pt] {};
			\node[fit=(b1)(b2)(b3),rounded corners,fill=green!18,inner xsep=-4pt,
	inner ysep=5pt,
			rotate=28] {};
			\node[fit=(c1)(c2)(c3),rounded corners,fill=cyan!18,inner xsep=-4pt,
	inner ysep=4pt,
			rotate=152] {};
		\end{pgfonlayer}
		\end{tikzpicture}
		\caption{Unconstrained}
		\label{fig:policy}
	  \end{subfigure} \hspace{0.2cm}
	\hfill
	\begin{subfigure}[t]{0.36\textwidth}
		\centering
		\begin{tikzpicture}[scale=1]
			\tikzstyle{vertex}=[circle,fill,scale=0.4,draw]
			\tikzstyle{special vertex}=[circle,fill=red!50,scale=0.3]
			\tikzstyle{square vertex}=[rectangle,fill,scale=0.5,draw]
			\tikzstyle{diamond vertex}=[regular polygon,regular polygon
			sides=3,rotate=45,fill,scale=0.3,draw]
			\node[vertex] at (1,2.598076) (a1) {};
			\node[special vertex] at (1.5,2.598076) (a2) {\faWifi};
			\node[vertex] at (2,2.598076) (a3) {};
			\node[vertex] at (0.25,1.299037) (b1) {};
			\node[special vertex] at (0.5,0.866025) (b2) {\faWifi};
			\node[vertex] at (0.75, 0.433013) (b3) {};
			\node[vertex] at (2.25,0.433013) (c1) {};
			\node[vertex] at (2.5,0.866025) (c2) {};
			\node[vertex] at (2.75,1.299037) (c3) {};
			%%%%%%%
			\node (r1) at (1.5,3.098076) [] {\faAndroid$_{_1}$};
			\node (r2) at (0.1464466,0.51247) [] {\faAndroid$_{_2}$};
			\node (r3) at (2.85355, 0.51247) [] {\faAndroid$_{_3}$};
			%%%%%%%
			\draw[thick,blue!80](a1) -- (b2);
			\draw[thick,red!30,dashed](a2) -- (b1);
			\draw[thick,blue!80](a2) -- (c2);
			\draw[thick,red!30,dashed](a2) -- (c3);
			\draw[thick,blue!80](b2) -- (c1);
			\draw[thick,red!30,dashed](c1) -- (a2);
			\draw[thick,red!30,dashed](b2) -- (a3);	  
			\draw[thick,black!40,dashed](b3) -- (c1);
			
			\begin{pgfonlayer}{background}
					\node[fit=(a1)(a2)(a3),rounded corners,fill=violet!15,inner xsep=3pt,
			inner ysep=4pt] {};
					\node[fit=(b1)(b2)(b3),rounded corners,fill=green!18,inner xsep=-4pt,
			inner ysep=5pt,
					rotate=28] {};
					\node[fit=(c1)(c2)(c3),rounded corners,fill=cyan!18,inner xsep=-4pt,
			inner ysep=4pt,
					rotate=152] {};
				\end{pgfonlayer}
			%\node [rounded corners, draw=red, rotate fit=45, fit=(b1) (b2) (b3)] {};
			%	\begin{pgfonlayer}{background}
			%		\node[fit=(a1)(a2)(a3),rounded corners,fill=violet!15,inner xsep=3pt,
			%inner ysep=4pt] {};
			%		\node[fit=(b1)(b2)(b3),rounded corners,fill=green!18,inner xsep=-4pt,
			%inner ysep=5pt,
			%		rotate=28] {};
			%		\node[fit=(c1)(c2)(c3),rounded corners,fill=cyan!18,inner xsep=-4pt,
			%inner ysep=4pt,
			%		rotate=152] {};
			%	\end{pgfonlayer}
		\end{tikzpicture}
		\caption{$(\kcomm,\kcomp) = (2,3)$}
		\label{fig:codesign}
	\end{subfigure}
	\hspace*{\fill}%
	\caption{\small (a) An example exchange graph
	$\Gcal$ in a $3$-rendezvous where each robot owns three observations
	(vertices). 
	Each potential inter-robot loop closure can be verified, if at least one
	robot shares its observation with the other robot.
%	To verify a potential inter-robot loop closure between two connected
%	vertices, at least one robot needs to share its observation with the other
%	robot. 
	(b) In the absence of any resource constraint, the robots can collectively
	verify all potential loop closures. The optimal lossless exchange policy
	\cite{Giamou18_ICRA} corresponds to sharing vertices of a
	minimum vertex cover, in this case
	the $3$ vertices marked in red.  (c) Now if the robots are only
	permitted to exchange at most $2$ vertices ($\kcomm = 2$) and verify at most
	$3$ edges ($\kcomp = 3$), they must decide which subset of observations to
	share (marked in red), and which subset of potential loop closures to verify
	(marked in blue). Note that these two subproblems are tightly coupled, as
	the selected edges must be covered by the selected vertices.
%	A lossless exchange policy in which the observations
%	associated to the vertices marked in red are transmitted. In the optimal
%	exchange policy, robots must exchange $3$ observations to cover all
%	potential loop closures. Now if robots are only permitted to exchange at
%	most $b = 2$ observations, they must decide which budget-feasible subset of
%	potential loop closures is most valuable (in expectation) based a
%  task-oriented objective.
  	}
	\label{fig:diagram}
%\vspace{-0.5cm}
\end{figure}

%Searching for inter-robot loop closure requires inter-agent
%communication, and thus can only occur during rendezvous. 
We consider the distributed loop closure detection problem 
during a multi-robot rendezvous.
Formally, an
\mbox{$r$-rendezvous} \cite{tian18} refers to a configuration where $r \geq 2$
robots are situated such that every robot can receive data broadcasted
by every other robot in the team. Each robot arrives at the rendezvous 
with a collection of sensory observations (e.g., images or laser scans) acquired throughout its
mission at different times and locations. 
Our goal is to discover associations (loop closures) between observations owned by different robots.
A \emph{distributed} framework for inter-robot loop closure detection divides
the computational burden
among different robots,
and thus enjoys several advantages
over centralized schemes (i.e., sending all observations to one node) including
reduced data transmission, improved flexibility, and robustness; see, e.g.,
\cite{Giamou18_ICRA,cieslewski2017efficient}. In these frameworks, robots
first exchange a compact representation (\emph{metadata} in \cite{Giamou18_ICRA})
of their observations in the form of, e.g., appearance-based feature vectors \cite{cieslewski2017efficient,Giamou18_ICRA}
or spatial clues (i.e., estimated location with uncertainty) \cite{Giamou18_ICRA}.
The use of metadata can help robots to efficiently search for identify a set of
\emph{potential} inter-robot loop closures by searching among their collected
observations.
For example, in many state-of-the-art SLAM systems (e.g., \cite{murORB2}), 
a pair of observations is declared a potential loop closure if the similarity score computed from matching the corresponding metadata 
(e.g., bag-of-words vectors) is above a threshold.\footnote{
As the fidelity of metadata increases, 
the similarity score becomes more effective in identifying true loop closures.
However, this typically comes at the cost of increased data
transmission during metadata exchange. In this paper we do not take into
account the cost of forming the exchange graph which is inevitable for optimal
data exchange \cite{Giamou18_ICRA}.}
The similarity score can also be used to \emph{estimate} the probability that a potential match is a true loop closure.
%This can be done, e.g., by learning the conditional distribution of true loop closures given the similarity scores
%in an offline stage. 
The set of potential loop closures identified from metadata exchange can be naturally represented using an $r$-partite
\emph{exchange graph} \cite{tian18,Giamou18_ICRA}.

\begin{definition}[Exchange Graph]
  \label{def:exchange_graph}
  \normalfont
  An exchange graph \cite{tian18,Giamou18_ICRA} between $r$ robots is a simple
  undirected $r$-partite graph $\Gall = (\Vall,\Eall)$ where each vertex $v
  \in \Vall$ corresponds to an observation collected by one robot at a
  particular time. The vertex set can be partitioned into $r$
  (self-independent) sets $\Vall = \VV_1 \uplus \cdots \uplus \VV_r$.  Each
  edge $\{u,v\} \in \Eall$ denotes a \emph{potential} inter-robot loop
  closure 
  identified by matching the corresponding metadata (here, $u \in \VV_i$ and $v \in \VV_j$).
%  between the corresponding observations ($u \in \VV_i$ and $v \in \VV_j$). 
  $\Gall$ is endowed with $w : \Vall \to \Rset_{>0}$ and $p : \Eall
\to [0,1]$ that quantify the size of each observation (e.g.,
bytes, number of keypoints in an image, etc), and the probability that an edge
corresponds to a true loop closure (independent of other edges), respectively.
\end{definition}

Even with high fidelity metadata, the set of potential loop closures $\Eall$ typically contains 
many false positives. Hence, it is essential that the robots collectively \emph{verify}
all potential loop closures. Furthermore, for each true loop closure that passes the verification step,
we also need to compute the relative transformation between the corresponding poses for the purpose of e.g., 
pose-graph optimization in CSLAM. 
For visual observations, these can be done by performing the so-called \emph{geometric verification},
which typically involves RANSAC iterations to obtain keypoints correspondences and an initial transformation,
followed by an optimization step to refine the initial estimate; see e.g., \cite{murORB2}.
%Similarly, for laser scans, variants of iterative closest point (ICP) algorithm can be used 
%for 3D scan registrations. 
Although geometric verification can be performed relatively efficiently, it can still become the computational 
bottleneck of the entire system in large problem instances \cite{heinly2015,Raguram2012BMVC}. In our case,
this corresponds to a large exchange graph (e.g., due to perceptual ambiguity and infrequent rendezvous). 
Verifying all potential loop closures in this case can exceed resource budgets on, e.g., energy consumption or CPU time.
Thus, from a resource scheduling perspective, it is natural to select an information-rich subset of 
potential loop closures $\Ecal \subseteq \Eall$ for verification.
Assuming that the cost of geometric verification is uniform among different
potential matches, 
we impose a computation budget by requiring that the selected subset (of edges
in $\Gall$) must contain no more than $k$ edges,
i.e., $|\Ecal| \leq k$.

In addition to the computation cost of geometric verification, robots also incur communication cost
when establishing inter-robot loop closures. More specifically,
before two robots can verify a potential loop closure, at least one of them must \emph{share} its observation with the
other robot. It has been shown that the minimum
data transmission required to verify any subset of edges $\EE \subseteq \Eall$
is determined by the minimum weighted vertex cover of the subgraph induced by
$\EE$ \cite{Giamou18_ICRA,tian18}.\footnote{Selecting a vertex is equivalent to
broadcasting the corresponding observation; see Figure~\ref{fig:policy}.} 
Based on this insight, we consider
three different models for communication budgets in Table~\ref{tab:gcomm}.
First, in Total-Uniform (\TotalUni{}) robots are allowed to exchange at most $b$
observations. This is justified under the assumption of uniform vertex weight (i.e. observation size) $w$. This assumption is relaxed in
Total-Nonuniform (\TotalNon{}) where total data transmission must be at most
$b$. Finally, in Individual-Uniform (\IndivUni{}), we assume $\Vall$ is
partitioned into $p$ blocks and robots are allowed to broadcast at most $b_i$
observations from the $i$th block for all $i \in [p]$. A natural partitioning of
$\Vall$ is given by $\VV_1 \uplus \cdots \uplus \VV_r$. In this case, \IndivUni{}
permits robot $i$ to broadcast at most $b_i$ of its observations for all
$i \in [r]$. This model captures the heterogeneous nature of the team.

Given the computation and communication budgets described above, robots must decide
which budget-feasible subset of potential edges to verify in order to maximize a
collective performance metric $\fe : 2^{\Eall} \to \Rset_{\geq 0}$.
%``best'' budget-feasible subset of potential loop closures (in
%expectation) that is feasible with respect
%to the abovementioned budget constraints. This is done
%by maximizing a collective objective $\fe : 2^{\Eall} \to \Rset_{\geq 0}$.
$\fe(\EE)$ aims to quantify the \emph{expected utility} gained by verifying
the potential loop closures in $\EE$. For concrete examples of
$\fe$, see Sections~\ref{sec:modular} and~\ref{sec:submodular} where we 
introduce three choices borrowed from \cite{tian18}.
This problem is formally defined below.
%The \emph{codesign} problem is formally defined below.

\begin{table}[t]
	\setlength{\tabcolsep}{13pt}
	\renewcommand{\arraystretch}{1.5}
	\caption{A subset of edges is feasible with regards to
	  communication budget if there exists a $\Vcal \subseteq \Vall$ that covers
	  that subset and $\Vcal$ satisfies the corresponding constraint; see \ref{eq:codesign}.}
	\centering
	\begin{tabular}{c||ccc}
		\hline
		\hline
		Type & \textbf{TU}$_b$ & 
		\textbf{TN}$_b$ & \textbf{IU}$_{b_{1:p}}$ \\
		\hline\hline
		\multirow{2}{*}{Constraint} & $|\Vcal| \leq b$                & $\sum_{v \in
		\Vcal} w(v) \leq b$   & $|\Vcal \cap \VV_i| \leq b_i$ for
		$i \in [p]$\\
		 & Cardinality & Knapsack & Partition Matroid\\
		\hline
		\hline
	\end{tabular}
	\label{tab:gcomm}
	%\vspace{-0.6cm}
\end{table}

\begin{problem}
  Let \textbf{CB} $\in
  \{\TotalUni{}_{b},\TotalNon{}_{b},\IndivUni{}_{b_{1:p}}\}$.
\begin{equation}
	\normalfont
	\begin{aligned}
		& \underset{\EE \subseteq \Eall}{\text{maximize}}
		& & \fe(\EE)\\
		& \text{subject to}
		&&	|\EE| \leq k, && (\text{\color{black!70!white}
	  \# of verifications}) \\
		  &   &   & \exists \, \Vcal \in \Cover(\EE) \text{ satisfying
			\textbf{CB}}. && (\text{\color{black!70!white}
		  data transmission})
	\end{aligned}
	\label{eq:codesign}
	\tag{$\text{P}_1$}
\end{equation}
\label{prob:codesign}
\end{problem}
%\vspace{-0.5cm}

\ref{eq:codesign} generalizes NP-hard problems, and thus is NP-hard in general.
In particular, for an NMS $\fe$ and when $b$ or $b_i$'s are sufficiently large
(i.e., unbounded communication), \ref{eq:codesign} becomes an instance of
general NMS maximization under a cardinality
constraint.  Similarly, for an NMS $\fe$ and a sufficiently large $k$ (i.e.,
unbounded computation), this problem reduces to a variant of $b$-DEP
\cite[Section 3]{tian18}. For general NMS maximization under a cardinality
constraint, no polynomial-time approximation algorithm can provide a constant
factor approximation better than $1-1/e$, unless P$=$NP; see \cite{krauseSurvey}
and references therein for results on hardness of approximation.
This immediately implies that $1-1/e$ is also the approximation barrier for the
general case of~\ref{eq:codesign} (i.e., general NMS objective).
In Sections~\ref{sec:modular} and
\ref{sec:submodular} we
present approximation algorithms with provable performance guarantees for
variants of this problem.

%\begin{problem}[Computation Constrained]
%  \begin{equation}
%	\begin{aligned}
%	  & \underset{\Ecal \subseteq \Eall}{\text{maximize}}
%	  & & f(\Ecal)\\
%	  & \text{subject to}
%	  &&   |\Ecal| \leq \kcomp.
%	\end{aligned}
%	\tag{$P_\text{\tiny\faLaptop}$}
%	\label{prob:compdesign}
%  \end{equation}
%\end{problem}
%\begin{problem}[Communication Constrained]
%  \begin{equation}
%	\begin{aligned}
%	  & \underset{\Ecal \subseteq \Eall}{\text{maximize}}
%	  & & f(\Ecal)\\
%	  & \text{subject to}
%	  &&   g_\text{\tiny\faWifi}(\Ecal) \leq \kcomm.
%	\end{aligned}
%	\tag{$P_\text{\tiny\faLaptop}$}
%	\label{prob:compdesign}
%  \end{equation}
%\end{problem}
%\vspace{-0.1cm}
\section{Modular Performance Metrics}
%\vspace{-0.4cm}
\label{sec:modular}
In this section, we consider a special case of \ref{eq:codesign} where
$\fe$ is normalized, monotone, and \emph{modular}.  This immediately implies
that $\fe(\varnothing) = 0$ and $\fe(\EE) = \sum_{e \in \EE} \fe(e)$ for all
non-empty $\EE \subseteq \Eall$ where $\fe(e) \geq 0$ for all $e \in \Eall$. 
%We refer to this special case as the \emph{modular codesign} problem. 
Without loss of generality, we focus on the case where $\fe(\EE)$ gives the expected number of true
inter-robot loop closures within $\EE$, i.e., $\fe : \EE \mapsto
\mathbb{E}\,[\text{number of correct matches within $\EE$}] = \sum_{e \in \EE} p(e)
\label{eq:maxprob_fe}$; see Definition~\ref{def:exchange_graph} and \cite[Eq.~$4$]{tian18}. This generic objective
is applicable to a broad range of scenarios where maximizing the expected
 number of ``true associations'' is desired (e.g., in distributed
place recognition). 
%Note that the results established in this section readily extend to any other normalized, monotone, and modular objectives. 
\ref{eq:codesign} with modular objectives generalizes the maximum coverage problem on graphs, and thus is NP-hard in general; see \cite{tian18}.
%  With unbounded computational capacity (i.e., $k = \infty$), modular codesign
%  becomes an instance of 
%  Codesign with a normalized, monotone, and modular objective is NP-hard:
%  when $\kcomp = \infty$ (unbo modular codesign reduces to the classical
%  weighted maximum coverage problem \cite{tian18}).  

In what follows, we present
an efficient constant-factor approximation scheme for \ref{eq:codesign} with modular objectives under
the communication cost regimes listed in Table~\ref{tab:gcomm}.
Note that merely deciding whether a given $\EE \subseteq \Eall$ is
feasible for \ref{eq:codesign} under \TotalUni{} is an instance of vertex cover
problem \cite{tian18} which is NP-complete. 
%Our plan is as follows: (i) we first
%transform~\ref{eq:codesign} into a more convenient form, then (ii) show that the
%new problem admits constant-factor approximation (greedy) algorithms; finally
%(iii)
We thus first transform~\ref{eq:codesign} into the following nested optimization
problem:
\begin{equation}
	\normalfont
	\begin{aligned}
	&\underset{\Vcal \subseteq \Vall}{\text{maximize}} 
	&& 
%	\tcbhighmath[boxsep=0pt,boxrule=1pt,colback=white,colframe=blue!50!black]
	\boxed{\underset{\EE
	\subseteq \edg(\Vcal),|\EE|\leq k}{\max}
	\hspace{.1cm} \fe(\EE)} \\
	& \text{subject to}
	&& \Vcal \text{ satisfies } \textbf{CB}.
	&&&
	\end{aligned}
	\tag{$\text{P}_2$}
	\label{prob:codesign_nested}
\end{equation}
Let $\fv : 2^{\Vall} \to \mathbb{R}_{\geq 0}$ return the optimal value of the inner optimization problem (boxed term)
and define $\fv(\varnothing) = 0$.
Note that $\fv(\VV)$ gives the maximum expected number of true inter-robot loop
closures achieved by broadcasting the observations associated to $\VV$ and
verifying at most $k$ potential inter-robot loop closures. In contrast to \ref{eq:codesign},
in \ref{prob:codesign_nested} we explicitly maximize $\fe$ over both
vertices and edges. This transformation reveals the inherent structure of our problem;
i.e., one needs to jointly decide which observations (vertices) to share (outer
problem), and which potential loop closures (edges) to verify among the set of
verifiable potential loop closures given the shared observations (inner
problem). For a modular $\fe$, it is easy to see that the inner problem admits a trivial solution
and hence $\fv$ can be efficiently computed: if $|\edg(\Vcal)| >
k$, return the sum of top $\kcomp$ edge probabilities in $\edg(\Vcal)$;
otherwise return the sum of all probabilities in $\edg(\Vcal)$. 
%
%simply pick
%$\min(k,|\edg(\VV)|)$
%most-probable edges in $\edg(\VV)$. 
%Therefore $\fv$ can be computed in
%$O(|\edg(\VV)| \times \log k)$ time.
%This motivates the following objective function defined over the power set of graph \emph{vertices}:
%\begin{align}
%	\fv: 2^{\Vall} & \to \mathbb{R}_{\geq 0} \nonumber \\
%	 \Scal & \mapsto 
%	\begin{cases}
%		\underset{\EE \subseteq \edg(\Scal), |\EE| \leq k}{\max} \sum_{e \in
%		\EE} p(e) &
%		\edg(\Scal) \neq \varnothing\\
%		0           & \edg(\Scal) = \varnothing.
%	\end{cases}
%	\label{eq:codesign_maxprob_fv}
%\end{align}
%
%Given any $\Vcal \subseteq \Vall$, $\fv(\Vcal)$ gives the optimal value of the corresponding inner problem,
%i.e.,
%the maximum expected number of
%true associations that can be extracted by \emph{broadcasting} the observations in $\Vcal$
%and performing at most $\kcomp$ verifications.
%Again, $\fv$ can be trivially evaluated for any nonempty $\Vcal \subseteq \Vall$: 
\begin{theorem}
	\normalfont
	$\fv$ is normalized, monotone, and submodular.
	\label{th:fvNMS}
\end{theorem}
%
%Now let us rewrite \ref{prob:codesign_nested} by replacing the inner problem 
%with $\fv$:
%\begin{equation}
%	\normalfont
%	\begin{aligned}
%		& \underset{\Vcal \subseteq \Vall}{\text{maximize}}
%		  &   & \fv(\Vcal)                          \\
%		& \text{subject to}
%		  &   & \text{$\Vcal$ satisfies \textbf{CB}}.
%	\end{aligned}
%	\tag{$\text{P}_3$}
%	\label{prob:codesign_maxprob_pv}
%\end{equation}

Theorem~\ref{th:fvNMS} implies that \ref{prob:codesign_nested} is an instance
of classical monotone submodular maximization subject to a 
cardinality (\TotalUni{}), a knapsack
(\TotalNon{}), or a partition matroid (\IndivUni{}) constraint. These problems admit constant-factor approximation
algorithms. The best performance guarantee in all cases is $1-1/e$
(Table~\ref{tab:modular_apx}); i.e., in the worst case, the expected number of
correct loop closures discovered by such algorithms is no less than 
$1-1/e \approx 63\%$ of that of an optimal solution.
Among these algorithms, variants of the natural
greedy algorithm are particularly well-suited for our application due to their
computational efficiency and incremental nature; see also \cite{tian18}.  These
simple algorithms enjoy constant-factor approximation guarantees, albeit with a
performance guarantee weaker than $1-1/e$ in the case of \TotalNon{} and
\IndivUni{}; see the first row of Table~\ref{tab:modular_apx} and
\cite{krauseSurvey}. More precisely, under the \TotalUni{} regime, the standard
greedy algorithm that simply selects (i.e., broadcasts) the next remaining
vertex $v$ with the maximum marginal gain over expected number of true loop closures provides the
optimal approximation ratio \cite{nemhauser1978analysis}; see
Algorithm~\ref{alg:mgreedy} in \ref{app:algorithms}. 
A na\"{i}ve implementation of this algorithm requires $O(\kcomm\cdot |\Vall|)$ evaluations of $\fv$,
where each evaluation $\fv(\Vcal)$ takes $O(|\edg(\VV)| \times \log k)$ time.
We note that the number of evaluations can be reduced by using the
so-called lazy greedy method \cite{minoux1978accelerated, krauseSurvey}.
Under \TotalNon{}, the
same greedy algorithm, together with one of its variants that normalizes marginal
gains by vertex weights, provide a performance guarantee of $1/2 \cdot (1-1/e)$
\cite{leskovec2007cost}.
Finally, in the case of \IndivUni{}, selecting the next \emph{feasible} vertex
according to the standard greedy algorithm leads to a performance guarantee of $1/2$
\cite{fisher1978analysis}. The following theorem provides an
approximation-preserving reduction from \ref{eq:codesign} to
\ref{prob:codesign_nested}.

%Now let $\VV_\text{grd}$
%be the set of vertices selected by the appropriate variant of the greedy
%algorithm in \ref{prob:codesign_nested}.
%%under
%%\textbf{CB} $\in \{\TotalUni{}_{b},\TotalNon{}_{b},\IndivUni{}_{b_{1:p}}\}$. 
%We
%have $g(\VV_{\text{grd}}) \geq \alpha \, \OPT_2$ where $\OPT_2$ denotes the
%optimal value of \ref{prob:codesign_nested} and $\alpha$ is the corresponding
%performance guarantee. 
%
%\textbf{CB}
%Therefore we have
%$g_\text{grd}$ 

%Hereafter, we treat these variants of natural
%greedy algorithms as a single class of algorithms and refer to it as
%\textsc{\small Modular-Greedy} or \textsc{\small M-Greedy}. We assume that the specific
%implementation will be clear based on the communication cost regime we are
%using. 

\begin{theorem}
  \normalfont 
  Given \textsf{ALG}, an $\alpha$-approximation algorithm for
  \ref{prob:codesign_nested} for an $\alpha \in (0,1)$, the following is an $\alpha$-approximation
  algorithm for \ref{eq:codesign}:
  \begin{enumerate}
	\item Run \textsf{ALG} on the corresponding instance of
	  \ref{prob:codesign_nested} to produce $\VV$.
	\item If $|\edg(\Vcal)| >
	  k$, return $\kcomp$ edges with highest probabilities in
	  $\edg(\Vcal)$; otherwise return the entire $\edg(\Vcal)$.
  \end{enumerate}
%
%  Let $\tilde{\VV}$ be an approximate solution for \ref{prob:codesign_nested}
%  such that $g(\tilde{\VV}) \geq \alpha \, \OPT_2$. Moreover, let $\tilde{\EE}
%  \in \argmax_{\EE \subseteq \edg(\tilde{\VV})} \, f(\EE) \, \text{ s.t. } |\EE|
%  \leq k$. Then 
%  For each communication cost regime listed in Table~\ref{tab:modular_apx},
%  the corresponding \textsc{\small Modular-Greedy} algorithm is a constant-factor approximation
%  algorithm for modular codesign under that regime.
  \label{thm:modular_apx}
\end{theorem}
Theorem~\ref{thm:modular_apx} implies that any near-optimal solution for
\ref{prob:codesign_nested} can be used to construct an equally-good near-optimal solution for our original problem
\ref{eq:codesign} with the same approximation ratio.
The first row of Table~\ref{tab:modular_apx} summarizes the
approximation factors of the abovementioned greedy algorithms for various
communication regimes. Note that this reduction also holds for more sophisticated
$(1-1/e)$-approximation algorithms (Table~\ref{tab:modular_apx}).
%
%Recall that for all communication cost regimes there
%exist more sophisticated algorithms (e.g., see \cite{Calinescu2011}) which
%achieve the \mbox{$1-1/e$} approximation ratio (second row in
%Table~\ref{tab:modular_apx}). Although Theorem~\ref{tab:modular_apx} holds for
%those algorithms as well, in this work we focus on greedy algorithms due to
%their simplicity, computational efficiency, and incremental nature.

\begin{remark}
\citet{Kulik2009MCP} study the so-called \emph{maximum coverage with packing constraint} (MCP),
which includes \ref{prob:codesign_nested} under \TotalUni{} (i.e., cardinality constraint) as a special case. 
%Similar to \cite{Kulik2009MCP}, we use the key insight that modular codesign can be decomposed into a nested problem
%when designing approximation algorithms.
Our approach differs from \cite{Kulik2009MCP} in the following three ways. 
Firstly, the algorithm proposed in \cite{Kulik2009MCP} achieves a performance
guarantee of $1 - 1/e$ for MCP by applying partial enumeration, which is
computationally expensive in practice. This additional complexity is due to a
knapsack constraint on edges (``items'' according to \cite{Kulik2009MCP}) which
is unnecessary in our application. As a result, the 
standard greedy algorithm retains the optimal performance guarantee for
\ref{prob:codesign_nested} without
any need for partial enumeration.
Secondly, in addition to \TotalUni, we study other models of communication budgets (i.e., \TotalNon{} and \IndivUni{}),
which leads to more general classes of constraints (i.e., knapsack and partition
matroid) that MCP does not consider. Lastly, besides providing efficient
approximation algorithms for \ref{prob:codesign_nested}, we further demonstrate
how such algorithms can be leveraged to solve the original problem
\ref{eq:codesign}
by establishing an approximation-preserving reduction (Theorem~\ref{thm:modular_apx}).
%\citet{Kulik2009MCP} present constant-factor approximation algorithm for the
%problem of \emph{maximum coverage with packing constraint} (MCP), which
%includes \ref{prob:codesign_nested} as a special case. Their approach
%combines greedy selection with partial enumeration,
%making the algorithm unsuitable for robotic applications. We adopted a similar
%idea for defining $\fv$. However, we showed
%that partial enumeration is unnecessary in our case.  
%\Kcomment{rewrite \ldots}
%Moreover, we show that the
%reduction from \ref{eq:codesign} to \ref{prob:codesign_nested} is an
%instance of the so-called approximation factor preserving reduction
%\cite{tian18}; i.e., constant factor approximation algorithms for
%\ref{prob:codesign_nested} can be used as \emph{proxy} to obtain constant
%factor approximation algorithms for the original modular codesign problem
%\ref{eq:codesign}.  See proof of Theorem~\ref{thm:modular_apx} in the appendix.
\end{remark}

%%%%%%%%%%%%%%%%%%%%%%%%%%%%%%%%%%%%%%%%%%%%%%%%%%%%%%%%%%%%%%%
%%%%%%%%%%%%%%%%%%%%%%%%%%%%%%%%%%%%%%%%%%%%%%%%%%%%%%%%%%%%%%%
%%%%%%%%%%%%%%%%%%%%%%%%%%%%%%%%%%%%%%%%%%%%%%%%%%%%%%%%%%%%%%%

\begin{table}[t]
	\setlength{\tabcolsep}{13pt}
	\renewcommand{\arraystretch}{1.5}
	\caption{Approximation ratio for modular objectives}
	\centering
	\begin{tabular}{c||ccc}
		\hline
		\hline
		Algorithm & \TotalUni{} & \TotalNon{} & \IndivUni{} \\
		\hline\hline
		\textsc{\small Greedy} & $1-1/e$ \cite{nemhauser1978analysis} &  $1/2 \cdot
			(1-1/e)$ \cite{leskovec2007cost} & $1/2$ \cite{fisher1978analysis} \\ \hline
			Best  & $1-1/e$ \cite{nemhauser1978analysis} & 
			$1-1/e$ \cite{sviridenko2004note}  & $1-1/e$ \cite{Calinescu2011}\\
		\hline
		\hline
	\end{tabular}
	\label{tab:modular_apx}
	%\vspace{-0.4cm}
\end{table}

%\vspace{-0.5cm}
\section{Submodular Performance Metrics}
\label{sec:submodular}
Now we assume $\fe : 2^{\Eall} \to \Rset_{\geq 0}$ can be an arbitrary NMS
objective, while limiting our communication cost regime to \TotalUni{}.
%Hereafter, we call this problem \emph{submodular codesign}.
From Section~\ref{sec:problemDefinition} recall that this problem is NP-hard in
general.
To the best of our knowledge, no prior work exists on approximation algorithms
for Problem~\ref{prob:codesign} with a general NMS objective.  Furthermore, the
approach presented for the modular case cannot be immediately extended to the
more general case of submodular objectives (e.g., evaluating the
generalized $g$ will be NP-hard). It is thus unclear whether any constant-factor (ideally, $1-1/e$) 
approximation can be achieved for an arbitrary NMS objective.
%the case of submodular codesign. 
%\Ycomment{move this sentence?}
%Extending to the set of NMS functions allows us to consider a broader range of
%objectives.

Before presenting our main results and approximation algorithms, let us briefly
introduce two such objectives suitable for our application; see also \cite{tian18}.
The D-optimality design criterion (hereafter, D-criterion),
defined as the log-determinant of the Fisher information matrix, is a popular
monotone submodular objective originated from the theory of optimal experimental design
\cite{Pukelsheim1993}.
This objective has been widely used across many domains (including SLAM) and 
enjoys rich geometric and information-theoretic interpretations; see, e.g.,
\cite{Joshi2009}. Assuming random edges of $\Gcal$ occur independently (i.e.,
potential loop closures ``realize'' independently), one can use 
the approximate expected gain in the D-criterion \cite{carlone2017attention}
\cite[Eq.~$2$]{tian18} as the objective in \ref{eq:codesign}.
This is known to be NMS if the information matrix prior
to rendezvous is positive definite \cite{tian18,shamaiah2010greedy}. 
In addition, in the case of 2D SLAM,
\citet{kasra16wafr,kasra18ijrr} show that the expected weighted number of
spanning trees (or, tree-connectivity) \cite[Eq.~$3$]{tian18} in the underlying
pose graph provides a graphical surrogate for the D-criterion with the advantage
of being cheaper to evaluate and requiring no metric knowledge about robots'
trajectories \cite{tian18}. Similar to the D-criterion, the expected tree-connectivity is
monotone submodular if the underlying pose graph is connected before selecting any potential
loop closure \cite{kasra18ijrr,kasra16wafr}. 

%In this section, we again present algorithms with provable performance
%guarantees. However, we show that our approximation factors become functions of
%the specific problem instance.  Later on we discuss implications of this result
%in more detail. 

Now let us revisit \ref{eq:codesign}.
It is easy to see that for a sufficiently large communication budget
$b$, \ref{eq:codesign} becomes an instance of NMS maximization under only a
cardinality constraint on edges (i.e., computation budget). 
Indeed, when $k < b$, the communication constraint becomes
redundant---because $k$ or less edges can trivially be covered by selecting at most $k < b$
vertices. 
Therefore, in such
cases the standard greedy algorithm
operating on edges $\Eall$ achieves a $1-1/e$
performance guarantee \cite{nemhauser1978analysis}; see, e.g.,
\cite{kasra16wafr,kasra18ijrr,carlone2017attention} for similar works.
Similarly, for a sufficiently large computation budget $k$ (e.g., when $\kcomm <
\lfloor \kcomp / \Delta \rfloor$ where $\Delta$ is the maximum degree of
$\Gcal$), \ref{eq:codesign}
becomes an instance of budgeted data exchange problem for which there exists
constant-factor (e.g., $1-1/e$ under \TotalUni{}) approximation algorithms \cite{tian18}.
Such algorithms 
%
%Recently, we showed how similar algorithms can be designed for measurement selection under a communication budget, 
%by establishing a so-called approximation factor preserving reduction \cite{tian18}. 
%
greedily select vertices, i.e., broadcast the corresponding observations and
select all edges incident to them \cite{tian18}.
%Once a vertex is
%included, all its incident edges are selected \cite{tian18}.  
Greedy edge (resp., vertex) selection thus provides $1-1/e$ approximation
guarantees for cases where $b$ is sufficiently smaller than $k$ and vice versa.
One can apply these algorithms on arbitrary instances of
\ref{eq:codesign} by picking edges/vertices greedily and stopping whenever at least one of the
two budget constraints is violated. Let us call the corresponding algorithms
\textsc{\small Edge-Greedy} (\textsc{\small E-Greedy}) and 
\textsc{\small Vertex-Greedy} (\textsc{\small V-Greedy}), respectively; see
Algorithm~\ref{alg:egreedy} and Algorithm~\ref{alg:vgreedy} in
\ref{app:algorithms}. Moreover, let
\textsc{\small Submodular-Greedy} (\textsc{\small S-Greedy}) be the algorithm
according to which
one runs both \textsc{\small V-Greedy} and \textsc{\small E-Greedy} and returns the
best solution among the two. A na\"{i}ve implementation of \textsc{\small S-Greedy} 
thus needs $O(b \cdot |\Vall| + k \cdot |\Eall|)$ evaluations of the objective.
For the D-criterion and tree-connectivity, each evaluation in principle takes $O(d^3)$ time where $d$ denotes the total number of robot poses.
In practice, the cubic dependence on $d$ can be eliminated by leveraging the sparse structure of the global pose graph.
This complexity can be further reduced by cleverly reusing Cholesky factors
in each round and utilizing rank-one updates; see \cite{tian18,kasra18ijrr}.
Once again, the number of function calls can be reduced by using the lazy greedy method.
The following theorem provides a performance guarantee for \textsc{\small S-Greedy} in terms of $b$, $k$, and
$\Delta$.
 
%Although the
%abovementioned algorithms  are
%originally designed for these extreme special cases, their conservative variants
%can produce feasible solutions to the codesign problem.
%%but their conservative variants can also be used to obtain feasible solutions
%%to the codesign problem. 
%The idea is to \emph{control} the greedy selection
%to remain feasible with respect to both budgets. 
%In the case of greedily selecting edges 
%%---\textsc{\small Edge-Greedy} hereafter,
%(hereafter, \textsc{\small Edge-Greedy} or \textsc{\small E-Greedy}),
%we can guarantee feasibility by selecting $\min(\kcomm, \kcomp)$ edges; see Algorithm~\ref{alg:egreedy}. 
%In the worst case, the selected edges will be disjoint and incur a communication cost of 
%$\min(\kcomm, \kcomp) \leq \kcomm$. 
%Similarly, in the case of greedily selecting vertices (hereafter, \textsc{\small Vertex-Greedy} or \textsc{\small V-greedy}),
%feasibility is guaranteed if we select $\min(\kcomm, \lfloor \kcomp/\Delta \rfloor)$ vertices, where $\Delta$
%denotes the maximum vertex degree in the exchange graph $\Gcal$; see Algorithm~\ref{alg:vgreedy}.
%In this way, the number of edges selected by \textsc{\small V-Greedy} is upper bounded by $\lfloor \kcomp/\Delta \rfloor \cdot \Delta \leq \kcomp$. 
%Despite their conservative nature, both algorithms provide important theoretical guarantees, 
%detailed in the following theorem.

\begin{theorem}
  \normalfont
  \textsc{\small Submodular-Greedy} is an $\alpha(b,k,\Delta)$-approximation algorithm
  for \ref{eq:codesign} where
	$\alpha(b,k,\Delta) \triangleq 1-\exp\big(-\min\,\{1,\gamma\}\big)$ in which
	$\gamma \triangleq \max\,\{b/k,\lfloor k/\Delta \rfloor/b\}$.
%  \begin{enumerate}
%	\item \textsc{\small Edge-Greedy} (Algorithm~\ref{alg:egreedy}) is an $\alpha_e(\kcomm, \kcomp)$-optimal approximation algorithm, where
%	\vspace{-0.4cm}
%	\begin{align}
%			\alpha_e(\kcomm, \kcomp) \triangleq 1-\exp(-\min\,\{1,\frac{\kcomm}{\kcomp}\})
%			\label{eq:egrd_apx}
%	\end{align}
%	\vspace{-0.38cm}
%	\item Let $\Delta$ be the maximum vertex degree in $\Gcal$. \textsc{\small Vertex-Greedy} (Algorithm~\ref{alg:vgreedy}) is an $\alpha_v(\kcomm, \kcomp, \Delta)$-optimal approximation algorithm, where
%	\vspace{-0.10cm}
%	\begin{align}
%		\alpha_v(\kcomm, \kcomp, \Delta) & \triangleq 1-\exp(-\min\,\{1, \frac{\lfloor \kcomp / \Delta \rfloor}{\kcomm}\})
%		\label{eq:vgrd_apx}
%	\end{align}
%  \end{enumerate}
  \label{thm:submodular_apx}
\end{theorem}

The complementary nature of \textsc{\small E-Greedy} and \textsc{\small
V-Greedy} is reflected in the performance guarantee $\alpha(b,k,\Delta)$
presented above.  Intuitively, \textsc{\small E-Greedy} (resp., \textsc{\small
V-Greedy}) is expected to perform well when computation (resp., communication)
budget is scarce compared to communication (resp., computation) budget. It is
worth noting that for a specific instance of
$\Gcal$, the actual performance guarantee of \textsc{\small S-Greedy} can be higher than $\alpha(b,k,\Delta)$.
This potentially stronger performance guarantee can be
computed \emph{a posteriori}, i.e., after running \textsc{\small S-Greedy} on
the given instance of $\Gcal$; see Lemma~\ref{lem:egrd_apx} and
Lemma~\ref{lem:vgrd_apx} in \ref{app:proofs}.
As an example, Figure~\ref{fig:sgrd_apx_kitti} shows the \emph{a posteriori} 
approximation factors of \textsc{\small S-Greedy}
in the KITTI~00 dataset (Section~\ref{sec:experiments}) for each combination of
budgets $(\kcomm, \kcomp)$.
Theorem~\ref{thm:submodular_apx} indicates that reducing $\Delta$ enhances the
performance guarantee $\alpha(b,k,\Delta)$. This is demonstrated in
Figure~\ref{fig:sgrd_apx_kitti_maxdeg_5}: after capping $\Delta$ at $5$,
the minimum approximation factor increases from about $0.18$ to $0.36$.
In practice, $\Delta$ may be large due to high 
\emph{uncertainty} in the initial set of potential inter-robot loop closures;
e.g., in situations with high perceptual ambiguity, an observation could
potentially be matched to many other observations during the initial phase of
metadata exchange. 
This issue can be mitigated by bounding $\Delta$ or increasing the \emph{fidelity} of metadata.
\begin{figure}[t]
	\centering
	\hspace*{\fill}%
	\begin{subfigure}[t]{0.30\textwidth}
		\centering
		\includegraphics[width=\textwidth]{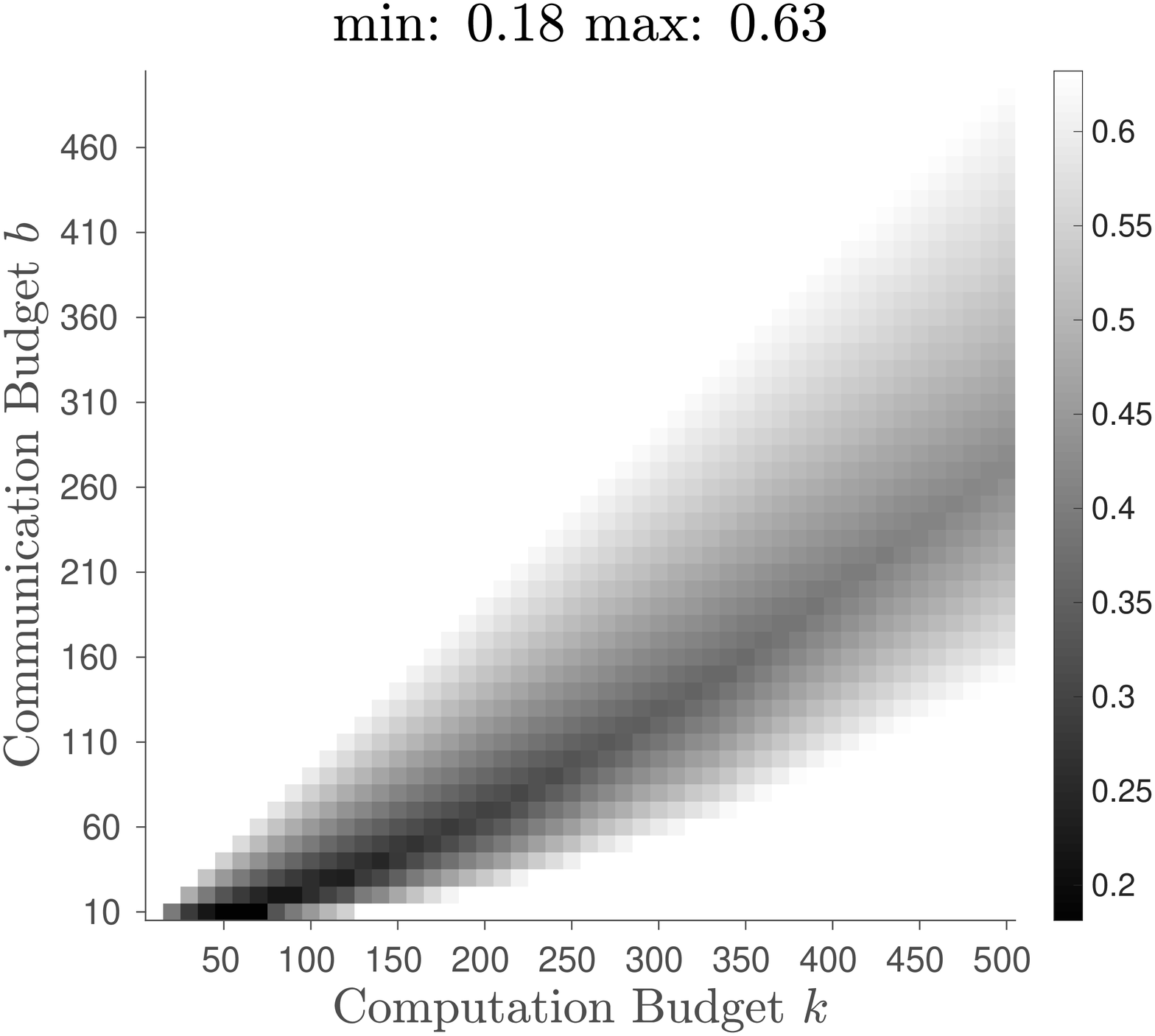}
		\caption{\small KITTI 00 ($\Delta=41$)}
		\label{fig:sgrd_apx_kitti}
	\end{subfigure}
	\hfill
	\begin{subfigure}[t]{0.30\textwidth}
		\centering
		\includegraphics[width=\textwidth]{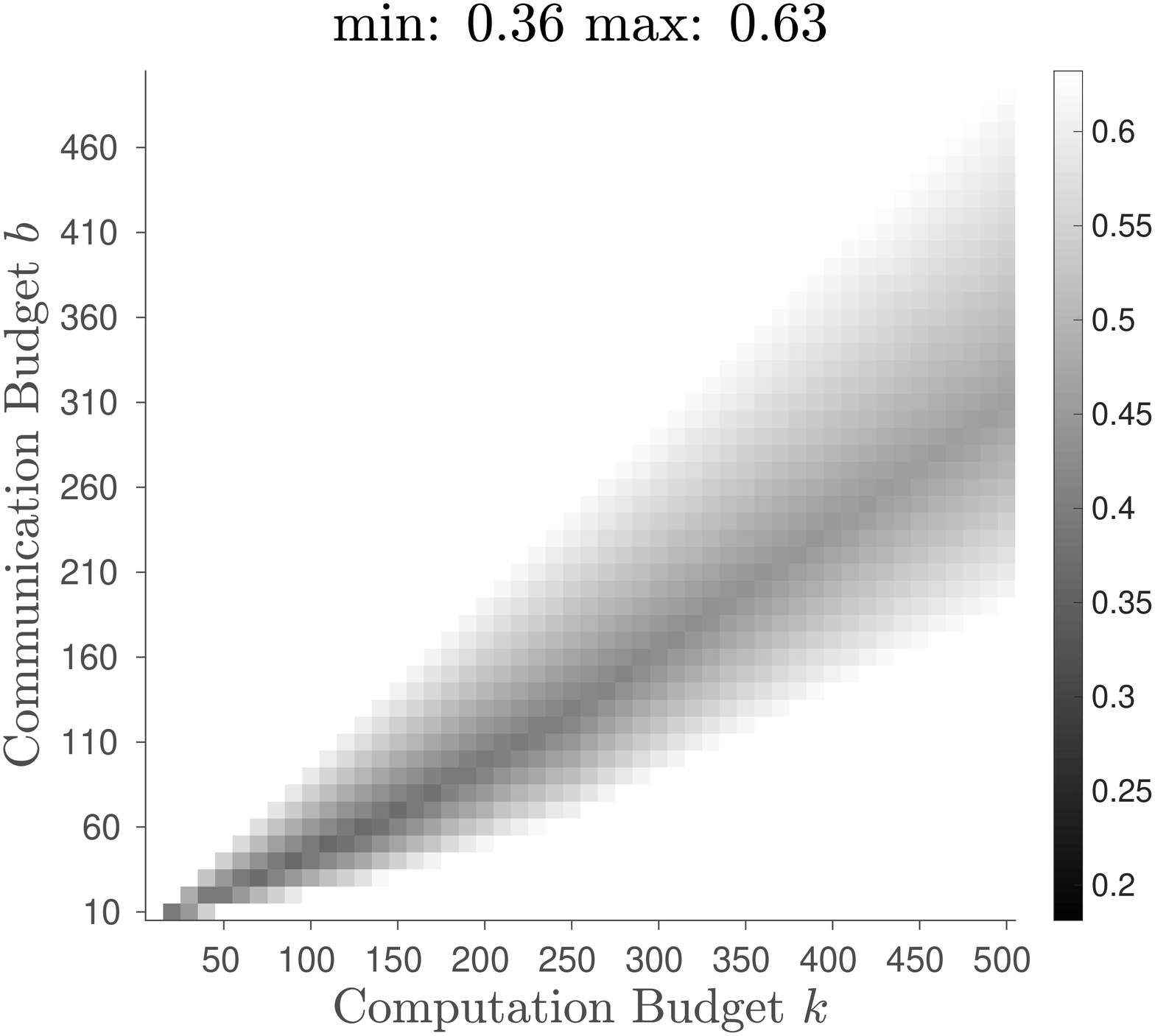}
		\caption{\small KITTI 00 ($\Delta=5$)} 
		\label{fig:sgrd_apx_kitti_maxdeg_5}
	\end{subfigure}
	\hfill
	\begin{subfigure}[t]{0.29\textwidth}
		\centering
		\includegraphics[width=\textwidth]{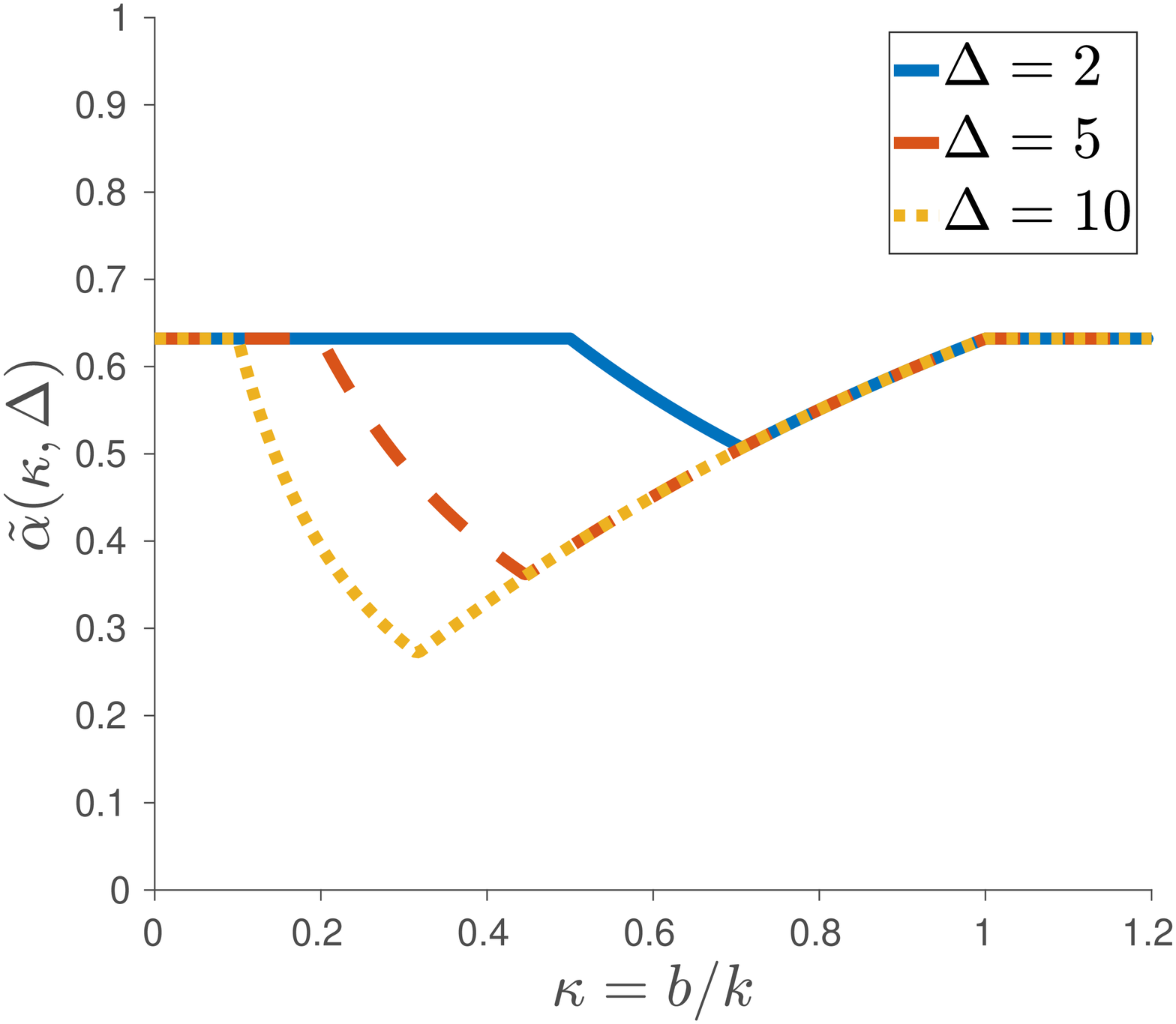}
		\caption{\small $\tilde{\alpha}(\kappa, \Delta)$}
		\label{fig:sgrd_apx_kappa}
	\end{subfigure}
	\hspace*{\fill}%
	\label{fig:sgrd_apx}
	\caption{\small  (a) \textsc{\small S-Greedy} approximation factor in
	KITTI~00 with $\Delta=41$. The approximation factor in this case varies between $0.18$ and $0.63$.
	(b) \textsc{\small S-Greedy} approximation factor in KITTI~00 with the maximum degree capped at $\Delta = 5$. 
	The approximation factor varies between $0.36$ and $0.63$.
	(c) The approximate \textsc{\small S-Greedy} performance guarantee as a function of $\kappa \triangleq \kcomm / \kcomp$ 
%	(assuming $\lfloor \kcomp / \Delta \rfloor \approx \kcomp / \Delta$) 
	with different $\Delta$.
  }
  %\vspace{-0.7cm}
\end{figure} 
To gain more intuition, let us approximate
$\lfloor k/\Delta \rfloor$ in $\alpha(b,k,\Delta)$ with $k/\Delta$.\footnote{This is a reasonable
  approximation when, e.g., $b$ is sufficiently large ($b \geq b_0$) since $k/(\Delta b) - 1/b
  < {\lfloor k/\Delta \rfloor}/{b} \leq k/( \Delta b)$ and thus the introduced
error in the exponent will be at most $1/b_0$.}
%
%
%
%$k/\Delta -1 < \lfloor k / \Delta \rfloor \leq k/\Delta$. We can
%approximate $\lfloor k/\Delta \rfloor \approx k/\Delta$ for a sufficiently large
%$b$ ($\exp(-1/b) \approx 1$).
%
%
%
%in the following we make the simplifying assumption that 
%$\lfloor \kcomp / \Delta \rfloor \approx \kcomp / \Delta$.\footnote{Note that this is done solely for visualization purpose.
%None of our theoretical guarantee depends on this assumption. \Ycomment{talk about degree variance...}}
With this simplification, the performance guarantee can be represented as
$\tilde{\alpha}(\kappa,\Delta)$, i.e., a function of the budgets ratio
$\kappa \triangleq \kcomm/\kcomp$ and $\Delta$. 
Figure~\ref{fig:sgrd_apx_kappa} shows $\tilde{\alpha}(\kappa,\Delta)$
%example $\alpha_s(\kcomm, \kcomp, \Delta)$ as a function of $\kappa$,
as a function of $\kappa$, 
with different maximum degree $\Delta$ (independent of any specific $\Gcal$).

\begin{remark}
  It is worth noting that
  $\alpha(b,k,\Delta)$ can be bounded from below by
  a function of $\Delta$, i.e., independent of $b$ and $k$. More precisely,
  after some algebraic manipulation,\footnote{Omitted due to space
  limitation.}
  it can be shown that $\alpha(b,k,\Delta) \geq 1 - \exp(-c(\Delta))$ where $ 1/(\Delta+1) \leq c(\Delta) \leq
  1/\sqrt{\Delta}$.
  This implies that
  for a bounded $\Delta \leq \Delta_\text{max}$, \textsc{\small S-Greedy} is a constant-factor
  approximation algorithm for \ref{eq:codesign}.
\end{remark}

\section{Experimental Results}
\label{sec:experiments}

%\vspace{-0.3cm}

\begin{figure}[t]
	\centering
	\hspace*{\fill}%
	\begin{minipage}{.38\linewidth}
	  \vspace{-1.7cm}
	  \caption{\small Left: KITTI 00; Right: 2D simulation \cite{tian18}. 
	  Each figure shows simulated trajectories of five robots. 
	  The KITTI trajectories shown are estimated purely using prior beliefs and odometry measurements 
	  (hence the drift). 
	The simulated trajectories shown are the exact ground truth.}
	\label{fig:basekkkkk}
  \end{minipage}
	\hfill%
	\begin{subfigure}[t]{0.28\textwidth}
		\centering
		\includegraphics[width=\textwidth]{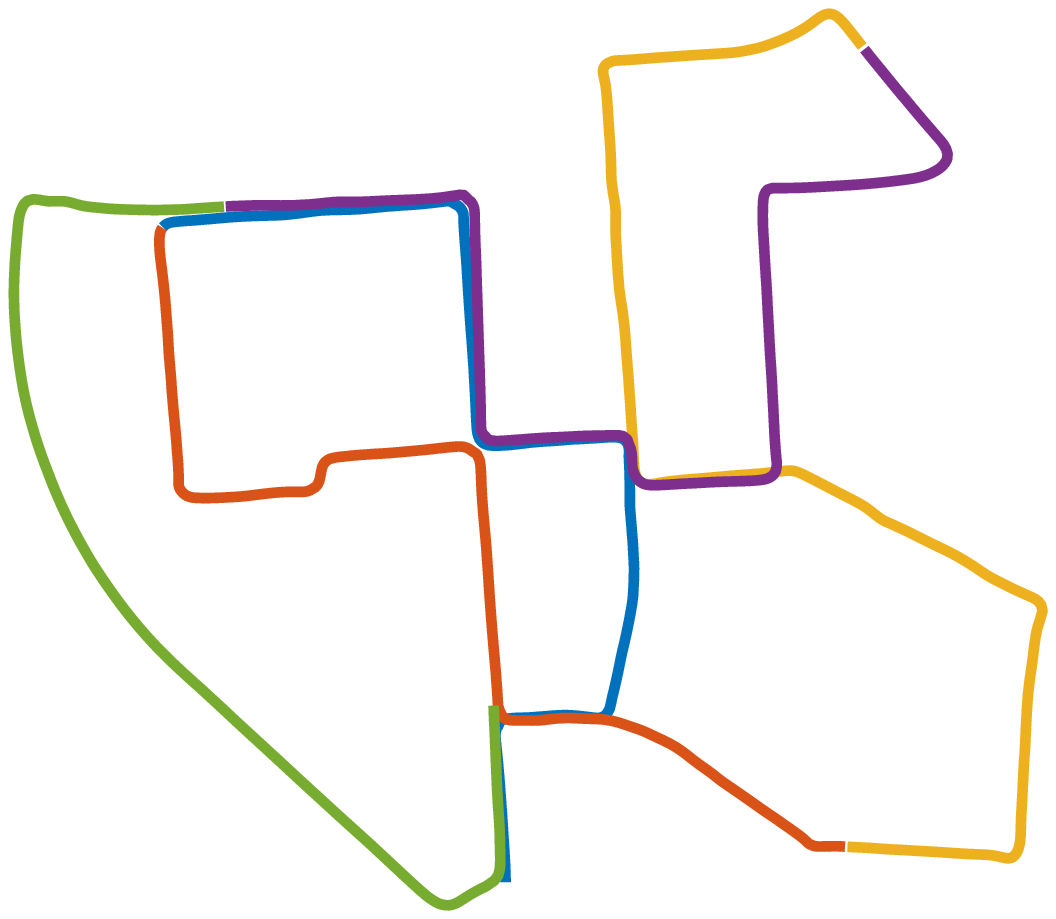}
		\caption{\small KITTI 00}
		\label{fig:KITTI_00_base_graph}
	\end{subfigure}
	\hfill
	\begin{subfigure}[t]{0.28\textwidth}
		\centering
		\includegraphics[width=\textwidth]{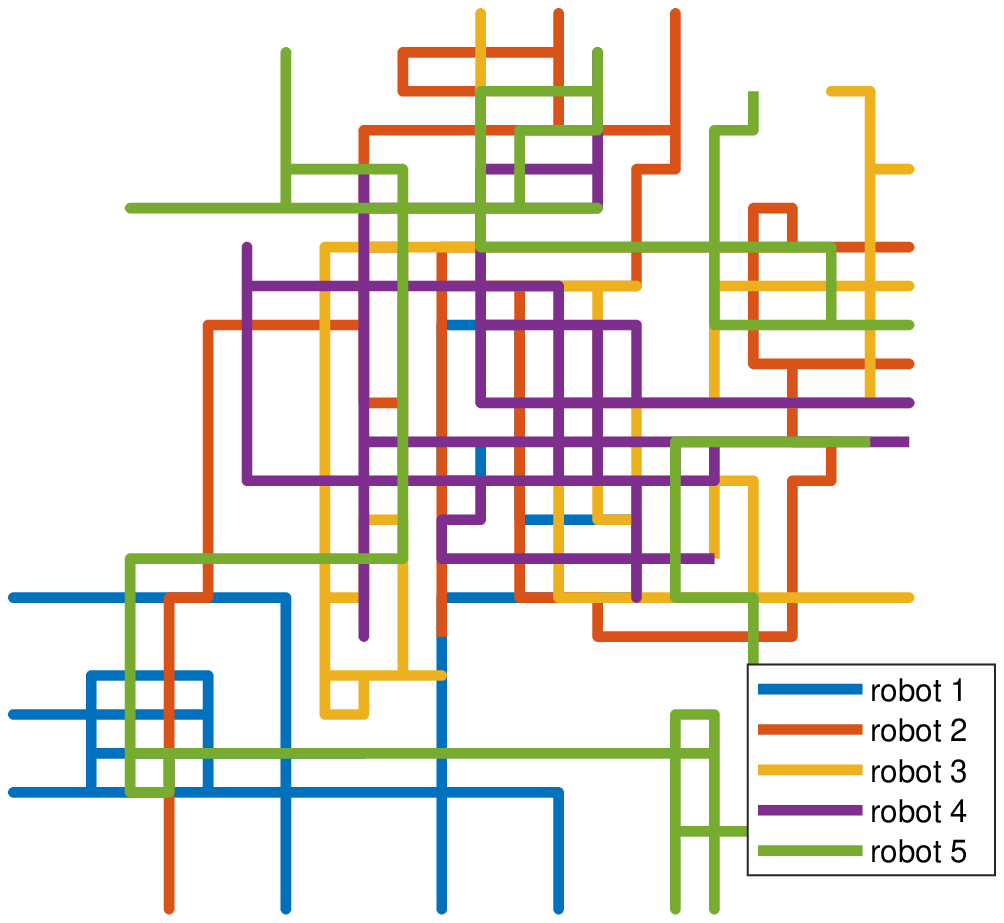}
		\caption{\small Simulation}
		\label{fig:Atlas_base_graph}
	\end{subfigure}
	\hspace*{\fill}%
	%\vspace{-0.7cm}
\end{figure}

We evaluate the proposed algorithms using sequence~00 of the KITTI odometry benchmark 
\cite{Geiger2012CVPR} and a synthetic Manhattan-like dataset \cite{tian18}. %\footnote{Due
%to space limit, we only discuss several important results in this section.
%Refer to the appendix for discussion of additional results.}
Each dataset is divided into multiple trajectories to simulate individual robots' paths (Figure~\ref{fig:basekkkkk}).
For the KITTI sequence, we project the trajectories to the 2D plane in order to use the tree-connectivity objective \cite{kasra16wafr}.
Visual odometry
measurements and potential loop closures are obtained from a modified version of
ORB-SLAM2 \cite{murORB2}. 
We estimate the probability of each potential loop closure by normalizing the corresponding DBoW2 score \cite{GalvezTRO12}.
For any specific environment, better mapping from the similarity score to the corresponding probability can be learned offline.
Nonetheless, these estimated probabilities are merely used to encourage the selection of more promising potential matches,
and the exact mapping used is orthogonal to the evaluation of the proposed approximation algorithms. 
%While our proposed algorithms---together with their performance guarantees---are independent of the exact values of the estimates,
%we note that better estimates can be acquired, e.g., by learning the conditional distribution of true loop
%closures given visual similarity scores in an offline stage.
For simulation, noisy odometry and loop closures are generated using the 2D
simulator of g2o \cite{kummerle2011g}. 
Each loop closure in the original dataset is considered a potential match
with an occurrence probability generated randomly according to the uniform
distribution $\Ucal(0,1)$.
%These probabilities are then used to sample the set of true loop closures
%from all potential matches. 
%Each potential loop closure is treated as a true loop closure with its corresponding occurrence probability.
Then, we decide if each potential match is a ``true'' loop closure by sampling from a Bernoulli distribution with the corresponding occurrence probability.
This process thus provides \emph{unbiased} estimates of the actual occurrence probabilities. 

In our experiments, each visual observation contains about $2000$ keypoints.
We ignore the insignificant variation in observation sizes and design all test cases
under the \TotalUni{} communication cost regime. 
Assuming each keypoint (consisting of a descriptor and coordinates) uses
$40$ bytes of data, a communication budget of $50$, for example, translates to
about $4$MB of data transmission \cite{tian18}. 
% Geometric verification and pose-graph optimization are done
%using an Intel Core i7 6700HQ (4 cores @ 2.6~GHz) with $16$GB RAM. With this
%relatively high compute power, a single geometric verification takes about
%$10$ms on average.

\subsection{Certifying Near-Optimality via Convex Relaxation}
\label{sec:certify}
Evaluating the proposed algorithms ideally requires access to the optimal value (OPT) of \ref{eq:codesign}. 
However, computing OPT by brute force is challenging even in relatively small instances. 
Following \cite{tian18}, we compute an upper bound UPT $\geq$ OPT by solving a natural convex relaxation of \ref{eq:codesign},
and use UPT as a surrogate for OPT. 
Comparing with UPT provides an \emph{a posteriori} certificate of near-optimality for solutions returned by the proposed algorithms.
Let $\ppp \triangleq [\pi_1,\dots,\pi_{n}]^\top$ and
$\bell  \triangleq [\ell_1,\dots,\ell_m]^\top$ 
be indicator variables corresponding to vertices and edges of $\Gcal$, respectively.
Let $\AAA$ be the undirected incidence matrix of $\Gcal$.
\ref{eq:codesign} can then be formulated in terms of $\ppp$ and $\boldsymbol{\ell}$.
For example, for modular objectives (Section~\ref{sec:modular}) under the \TotalUni{} communication model, 
\ref{eq:codesign} is equivalent to maximizing $\fe^\star(\boldsymbol{\ell}) \triangleq \sum_{e \in \Eall} p(e) \cdot \ell_e$
subject to $(\ppp, \bell) \in \Fcal_{\text{int}}$, where
$\Fcal_{\text{int}} \triangleq \big\{(\ppp,\bell) \in \{0,1\}^n \times \{0,1\}^m
:\mathbf{1}^{\hspace{-0.05cm}\top}\ppp\leq \kcomm, 
\mathbf{1}^{\hspace{-0.05cm}\top}\bell\leq \kcomp, 
\mathbf{A}^{\hspace{-0.1cm}\top}\ppp \geq \bell\big\}$.
Relaxing $\Fcal_{\text{int}}$ to 
$\Fcal\triangleq \big\{(\ppp,\bell) \in [0,1]^n \times [0,1]^m
:\mathbf{1}^{\hspace{-0.05cm}\top}\ppp\leq \kcomm,
 \mathbf{1}^{\hspace{-0.05cm}\top}\bell\leq \kcomp, 
\mathbf{A}^{\hspace{-0.1cm}\top}\ppp \geq \bell\big\}$
gives the natural LP relaxation, whose optimal value is an upper bound on OPT.
Note that in this special case, we can also compute OPT directly by solving the original ILP (this is not practical for real-world applications). 
On the other hand, for maximizing the D-criterion and tree-connectivity (Section~\ref{sec:submodular}),
convex relaxation produces determinant maximization (maxdet) problems \cite{vandenberghe1998determinant} subject to affine constraints.
In our experiments, all LP and ILP instances are solved using built-in solvers in MATLAB.
All maxdet problems are modeled using the YALMIP toolbox \cite{Lofberg2004} and solved using SDPT3 \cite{Toh99sdpt3} in MATLAB.

\subsection{Results with Modular Objectives}

%\vspace{-0.6cm}
\begin{figure}[t]
	\centering
	\includegraphics[width=0.8\textwidth]{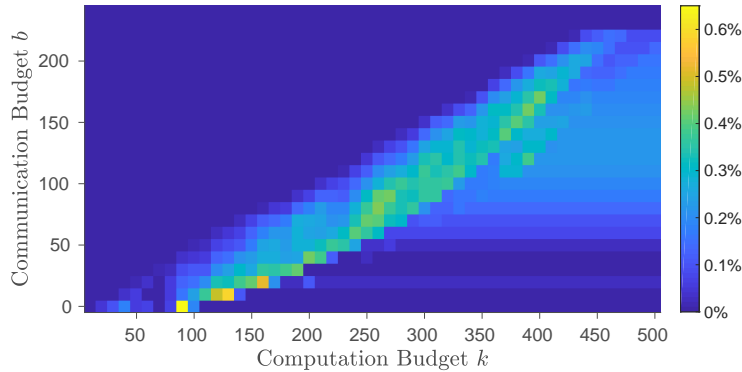}
	\caption{\small Optimality gap of \textsc{\small M-Greedy} in KITTI~00 under \TotalUni{}.
	In this case, the objective is to maximize the expected number of loop
	closures (\cite[Eq.~$4$]{tian18}).
	For each problem instance specified by a pair of budgets $(\kcomm, \kcomp)$, we calculate the difference between the achieved value and 
	the optimal value obtained by solving the ILP. 
	Each value is then normalized by the maximum achievable value given infinite budgets and shown in percentage.}
%	In all instances, the optimality gaps are close to zero.
%	Note that when $\kcomm > \kcomp$ (upper-left region), 
%	\textsc{\small M-Greedy} always finds the optimal solution.
%	This result is expected, as in this case \ref{eq:codesign} reduces to a simple knapsack problem under cardinality constraint,
%	for which the greedy algorithm is known to be optimal.}
	\label{fig:KITTI_maxprob_grd_vs_opt}
	%\vspace{-0.7cm}
\end{figure}

For brevity, we refer to the proposed greedy algorithm in Section~\ref{sec:modular} under \TotalUni{} (see Algorithm~\ref{alg:mgreedy}) as \textsc{\small M-Greedy}.
Figure~\ref{fig:KITTI_maxprob_grd_vs_opt} shows the optimality gap of \textsc{\small M-Greedy} evaluated on the KITTI~00 dataset.
In this case, the objective is to maximize the expected number of true loop closures (\cite[Eq.~$4$]{tian18}).
Given the exchange graph,
we vary the communication budget $\kcomm$ and computation budget $\kcomp$ to 
produce different instances of \ref{eq:codesign}. 
For each instance, we compute the difference between the achieved value and the optimal value 
obtained by solving the corresponding ILP (Section~\ref{sec:certify}).
The computed difference is then normalized by the maximum achievable value given infinite budgets and converted into percentage.
The result for each instance is shown as an individual cell in Figure~\ref{fig:KITTI_maxprob_grd_vs_opt}. 
In all instances, the optimality gaps are close to zero.
In fact, the maximum \emph{unnormalized} difference across all instances
is about $1.35$, i.e., the achieved value and the optimal value only differ by $1.35$ expected loop closures.
These results clearly confirm the near-optimal performance of the proposed algorithm.
Furthermore, in many instances, \textsc{\small M-Greedy} finds optimal solutions (shown in dark blue).
We note that this result is expected when $\kcomm > \kcomp$ (top-left region),
as in this case \ref{eq:codesign} reduces to a degenerate instance of the knapsack problem,
for which greedy algorithm is known to be optimal.

\subsection{Results with Submodular Objectives}

%\begin{figure}[t]
%	\centering
%	\hspace*{\fill}%
%	\begin{subfigure}[t]{0.30\textwidth}
%		\centering
%		\includegraphics[width=\textwidth]{figures/kitti_fim_b_20.eps}
%		\caption{\small $b = 20$}
%		\label{fig:kitti_fim_b_20}
%	\end{subfigure}
%	\hfill
%	\begin{subfigure}[t]{0.30\textwidth}
%		\centering
%		\includegraphics[width=\textwidth]{figures/kitti_fim_b_70.eps}
%		\caption{\small $b = 70$}
%		\label{fig:kitti_fim_b_70}
%	\end{subfigure}
%	\hfill
%	\begin{subfigure}[t]{0.30\textwidth}
%		\centering
%		\includegraphics[width=\textwidth]{figures/kitti_fim_b_150.eps}
%		\caption{\small $b = 150$}
%		\label{fig:kitti_fim_b_150}
%	\end{subfigure}
%	\hspace*{\fill}%
%	\caption{\small Performance of \textsc{\small Submodular-Greedy} in KITTI~00
%		with the D-criterion objective under the \TotalUni{} regime. \Ycomment{merge with tree-connectivity?}
%	}
%	\label{fig:kitti_fim_fix_b}
%\end{figure}
\begin{figure}[t]
	\centering
	\hfill
	\begin{subfigure}[t]{0.30\textwidth}
		\centering
		\includegraphics[width=\textwidth]{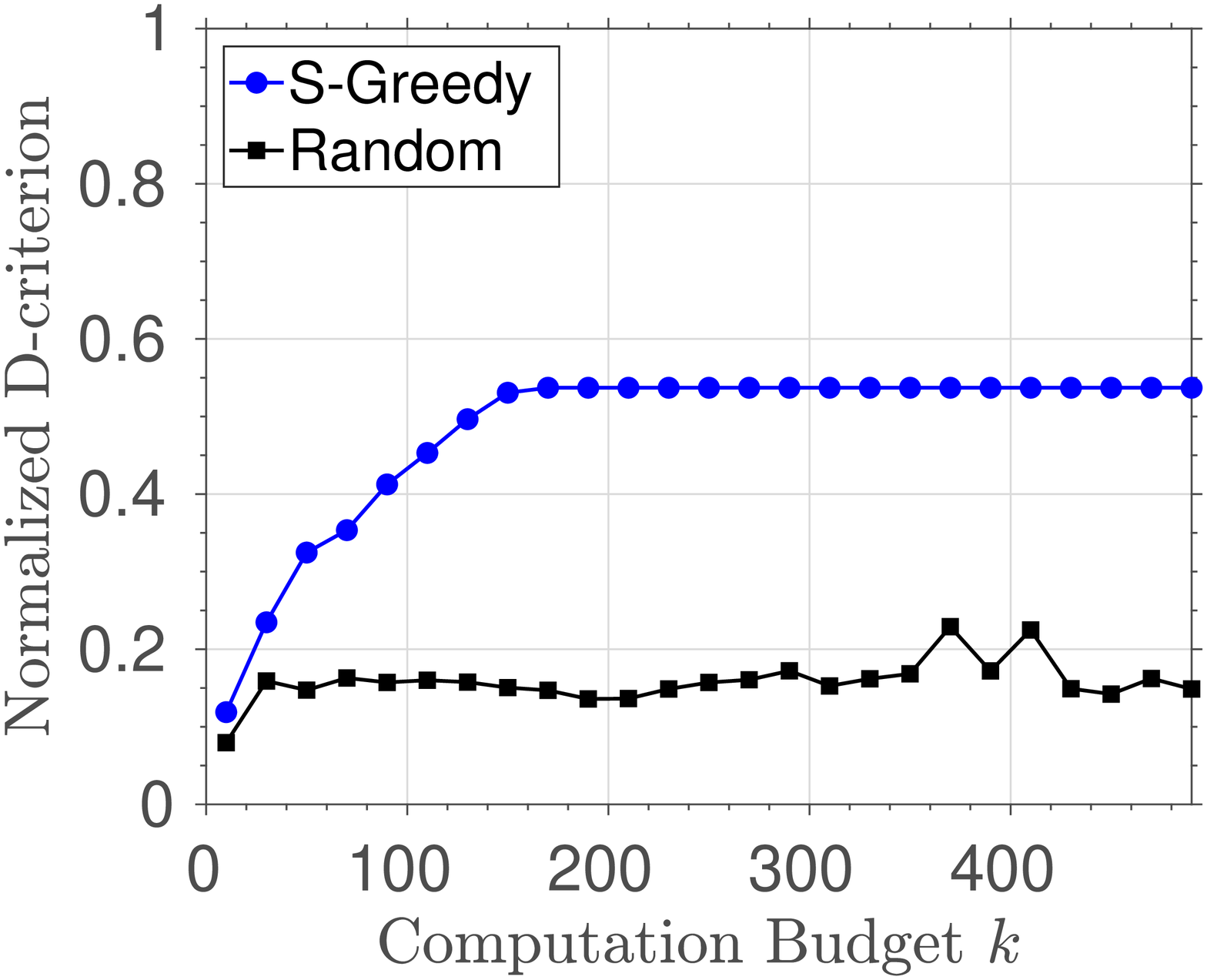}
		\caption{\small $b = 20$}
		\label{fig:kitti_fim_b_20}
	\end{subfigure}
	\hfill
	\begin{subfigure}[t]{0.30\textwidth}
		\centering
		\includegraphics[width=\textwidth]{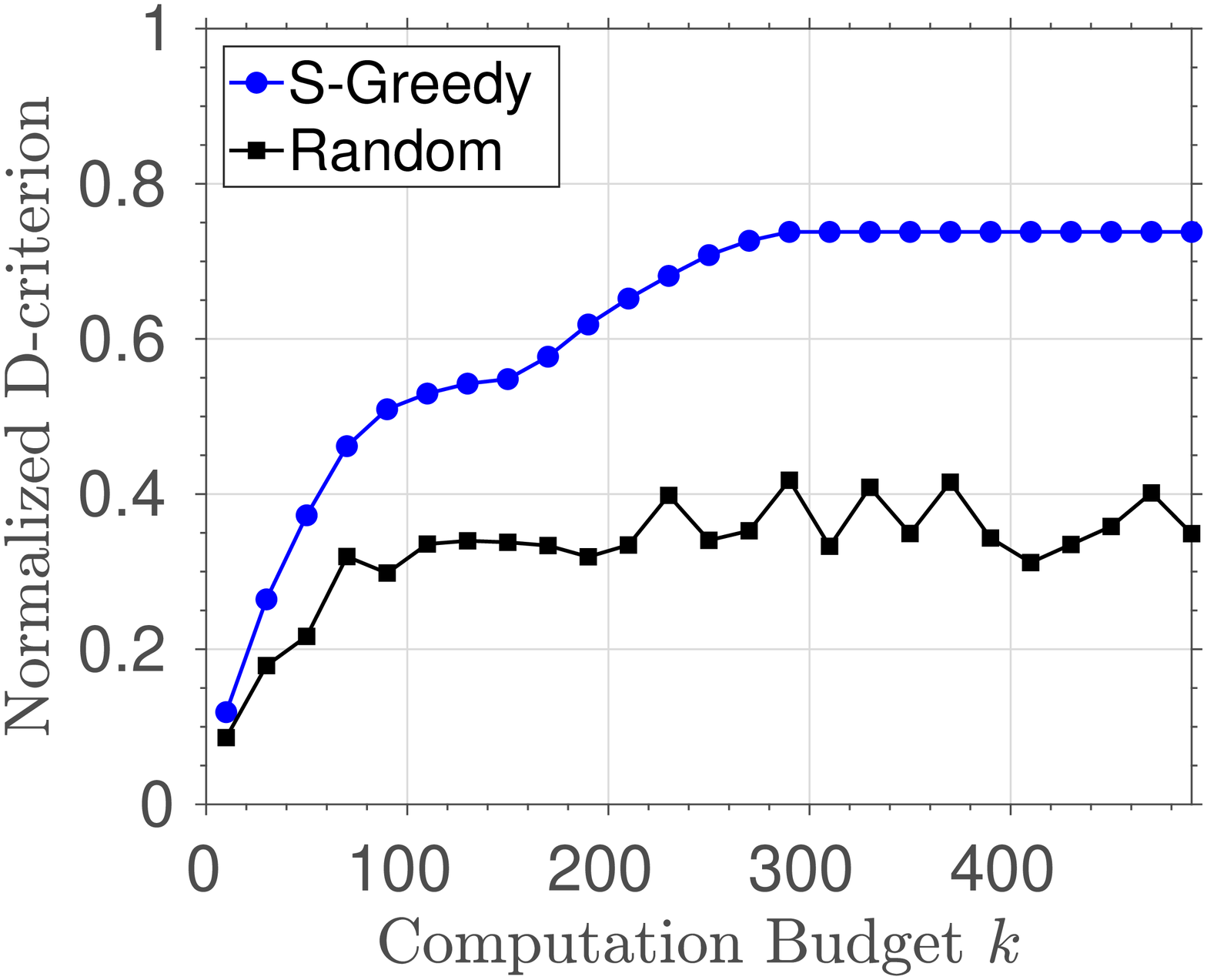}
		\caption{\small $b = 70$}
		\label{fig:kitti_fim_b_70}
	\end{subfigure}
	\hfill
	\begin{subfigure}[t]{0.30\textwidth}
		\centering
		\includegraphics[width=\textwidth]{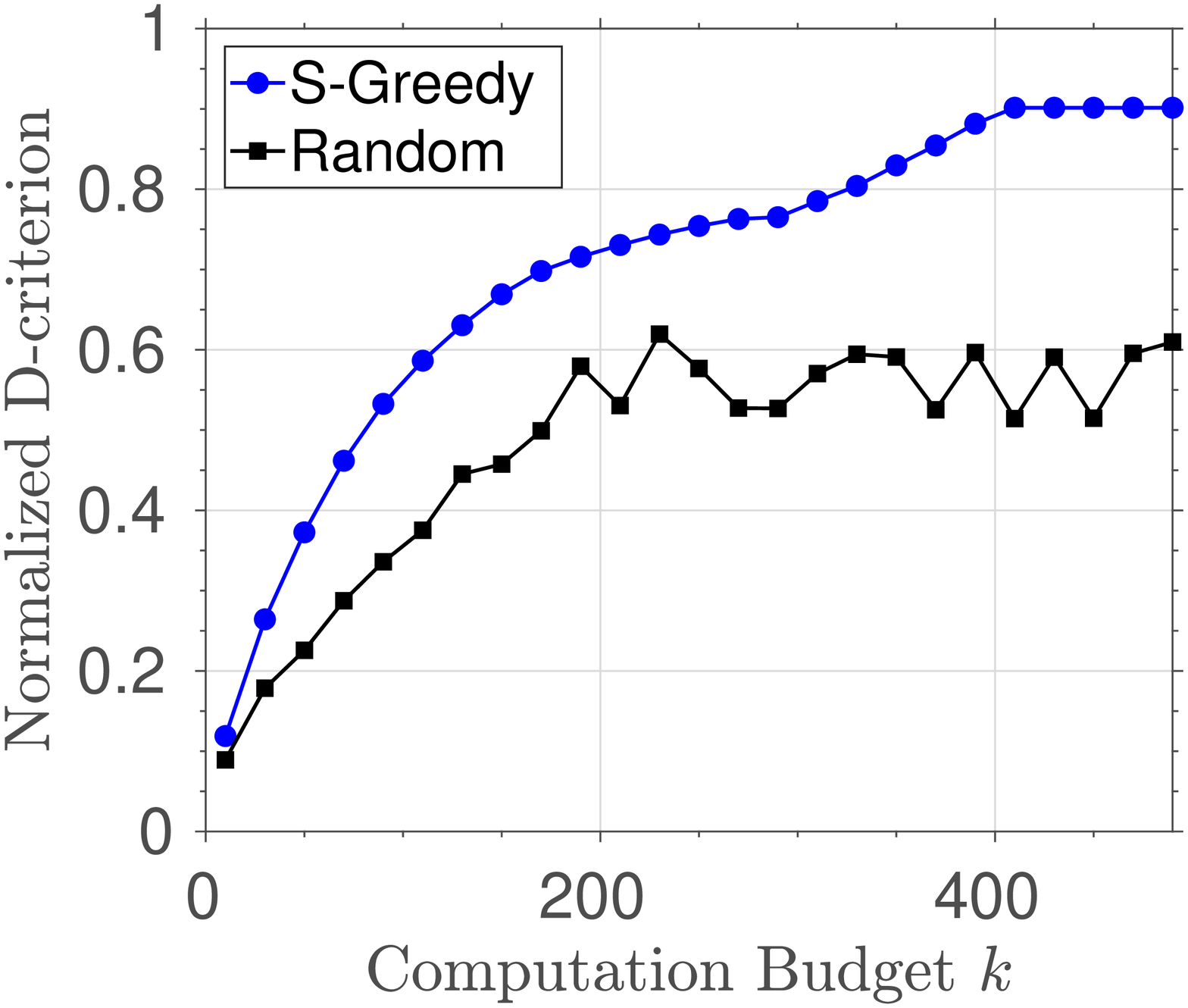}
		\caption{\small $b = 150$}
		\label{fig:kitti_fim_b_150}
	\end{subfigure}
	\\
	\centering
	\hfill
	\begin{subfigure}[t]{0.30\textwidth}
		\centering
		\includegraphics[width=\textwidth]{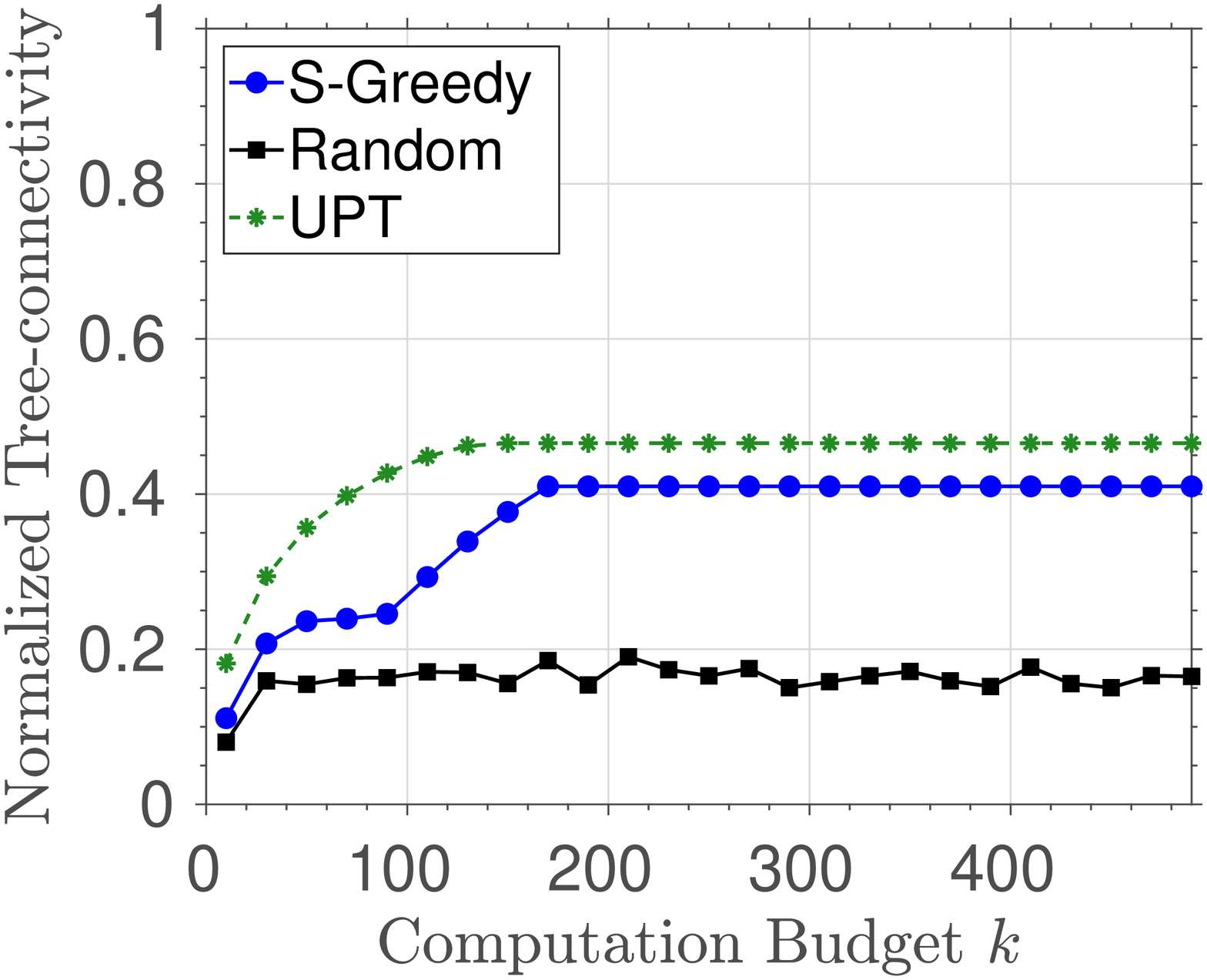}
		\caption{\small $b = 20$}
		\label{fig:kitti_wst_b_20}
	\end{subfigure}
	\hfill
	\begin{subfigure}[t]{0.30\textwidth}
		\centering
		\includegraphics[width=\textwidth]{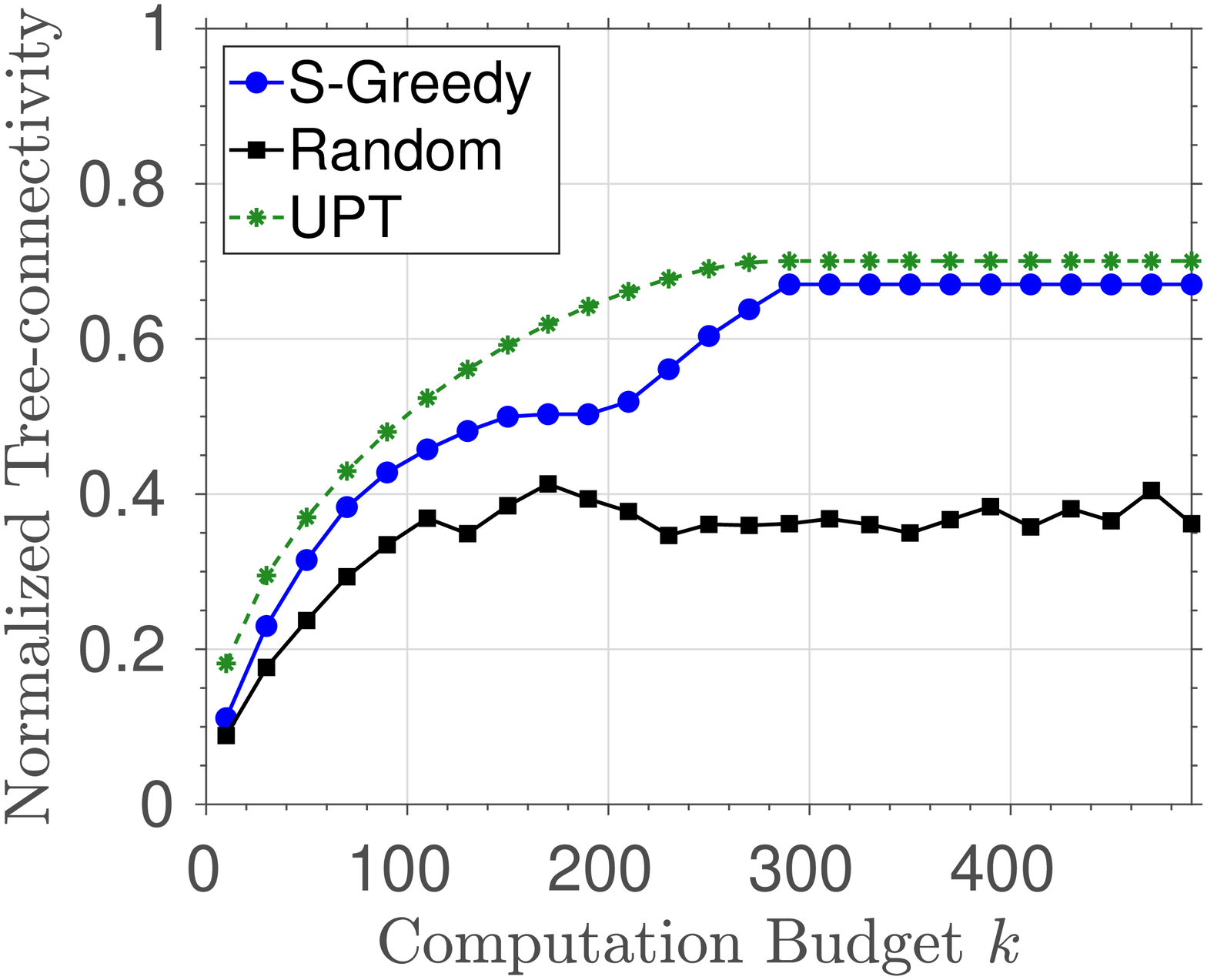}
		\caption{\small $b = 70$}
		\label{fig:kitti_wst_b_70}
	\end{subfigure}
	\hfill 
	\begin{subfigure}[t]{0.30\textwidth}
		\centering
		\includegraphics[width=\textwidth]{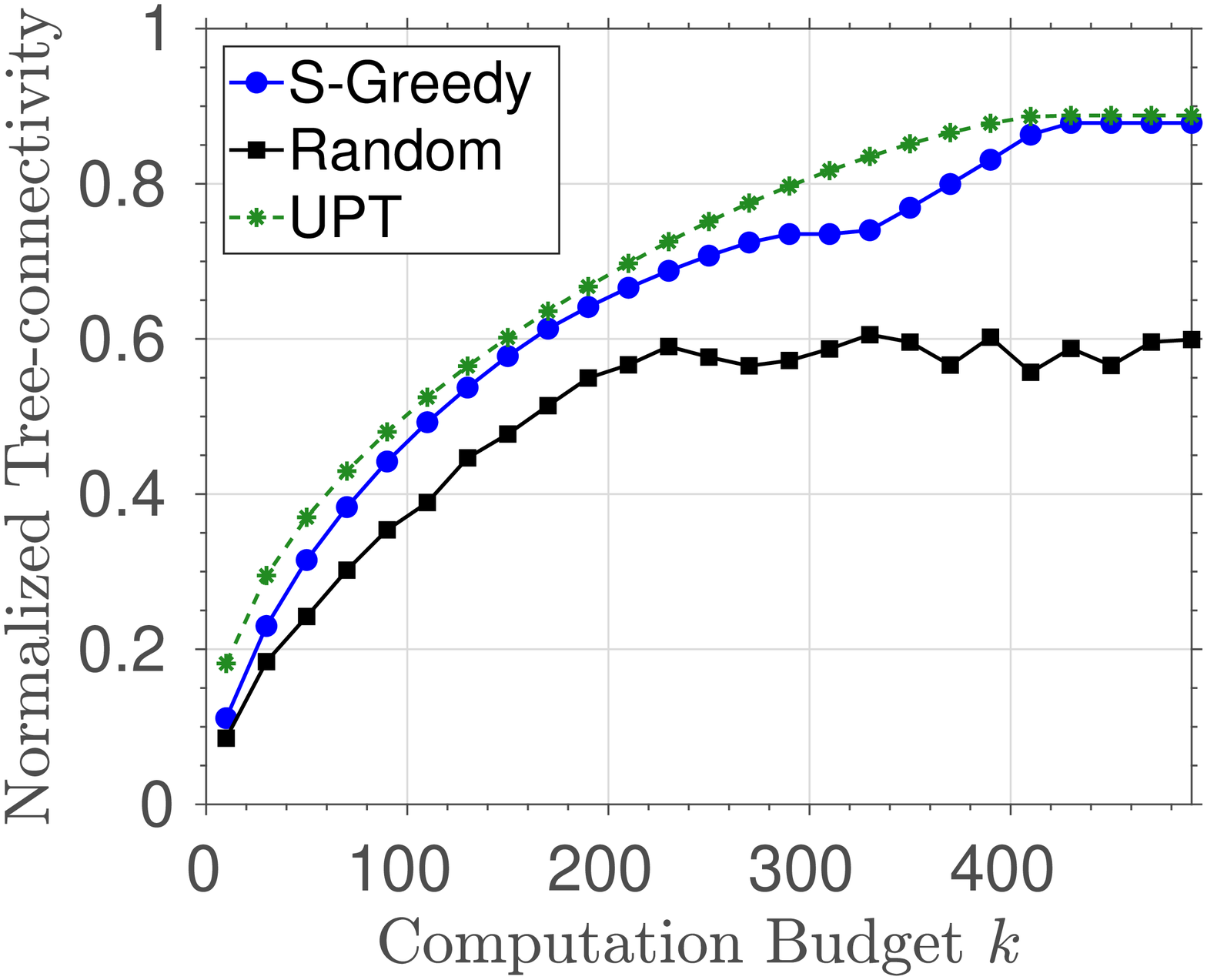}
		\caption{\small $b = 150$}
		\label{fig:kitti_wst_b_150}
	\end{subfigure}
	\caption{\small Performance of \textsc{\small Submodular-Greedy} in KITTI~00
	under \TotalUni{},
	with the D-criterion objective in (a)-(c) and tree-connectivity objective in (d)-(f).
	Each figure shows one scenario with a fixed communication budget $\kcomm$ and varying computation budget $\kcomp$. 
	The proposed algorithm is compared against a random baseline.
	With the tree-connectivity objective, we also show the upper bound (UPT) computed using convex relaxation.
	All values are normalized by the maximum achievable value given infinite budgets.
	}
	\label{fig:kitti_fix_b}
	%\vspace{-0.5cm}
\end{figure}

Figure~\ref{fig:kitti_fix_b} shows performance of \textsc{\small S-Greedy} in KITTI~00 
under the \TotalUni{} regime. 
Figures~\ref{fig:kitti_fim_b_20}-\ref{fig:kitti_fim_b_150} use the D-criterion objective \cite[Eq.~$2$]{tian18},
and \ref{fig:kitti_wst_b_20}-\ref{fig:kitti_wst_b_150} use the tree-connectivity objective \cite[Eq.~$3$]{tian18}.
Each figure corresponds to one scenario with a fixed communication budget $\kcomm$ and varying computation budget $\kcomp$.
We compare the proposed algorithm with a baseline that randomly selects 
$\kcomm$ vertices, and then selects $\kcomp$ edges at random from the set of edges incident to the selected vertices.
%vertices and edges such that the selected vertices cover the selected edges.
When using the tree-connectivity objective,
we also plot the upper bound (UPT) computed using convex relaxation (Section~\ref{sec:certify}).\footnote{For KITTI~00 we do not show UPT when using the D-criterion objective, because solving the convex relaxation in this case is too time-consuming.}
All values shown are normalized by the maximum achievable value given infinite budgets.
In all instances, \textsc{\small S-Greedy} clearly outperforms the random baseline. 
Moreover, in \ref{fig:kitti_wst_b_20}-\ref{fig:kitti_wst_b_150}, the achieved objective values are close to UPT.
In particular, using the fact that UPT $\geq$ OPT, we compute a lower bound on 
the empirical approximation ratio in \ref{fig:kitti_wst_b_20}-\ref{fig:kitti_wst_b_150},
which varies between $0.58$ and $0.99$.
This confirms the near-optimality of the approximate solutions found by \textsc{\small S-Greedy}.
The inflection point of each \textsc{\small S-Greedy} curve in Figure~\ref{fig:kitti_fim_b_20}-\ref{fig:kitti_wst_b_150}
corresponds to the point where the algorithm switches from greedily selecting edges (\textsc{\small E-Greedy}) to greedily selecting vertices (\textsc{\small V-Greedy}).
Note that in all scenarios, the achieved value of \textsc{\small S-Greedy} eventually saturates.
This is because
as $\kcomp$ increases, \textsc{\small S-Greedy}
eventually spends the entire communication budget and halts.
%Thus, additional computation resources become redundant.
In these cases, 
since we already select the maximum number (i.e., $\kcomm$) of vertices,
the algorithm achieves a constant approximation factor of $1 - 1/e$ (see Lemma~\ref{lem:vgrd_apx} in \ref{app:proofs}).
Due to space limitation, additional results on the simulated datasets 
and results comparing \textsc{\small M-Greedy} with \textsc{\small S-Greedy} in the case of modular objectives are reported in \ref{app:additional_results}. 

%\vspace{-0.5cm}
\section{Conclusion}
%\vspace{-0.3cm}

Inter-robot loop closure detection is a
critical component for many multirobot applications, including those that require
collaborative localization and mapping. The resource-intensive nature of
this process, together with the limited nature of mission-critical resources available
on-board necessitate intelligent utilization of available/allocated resources.
This paper studied distributed inter-robot loop closure detection under
computation and communication budgets. 
More specifically, we sought to maximize monotone submodular performance metrics by
selecting a subset of potential inter-robot loop closures that is feasible with
respect to both computation and communication budgets. 
This problem
generalizes previously studied NP-hard problems that assume either unbounded
computation or communication. 
%We presented approximation algorithms with provable
%performance guarantees for this problem. 

In particular, for monotone modular objectives (e.g.,
expected number of true loop closures)
we presented a family of greedy algorithms with constant-factor approximation
ratios under multiple communication cost regimes.
This was made possible through establishing an approximation-factor preserving
reduction to well-studied instances of monotone submodular maximization problems
under cardinality, knapsack, and partition matroid constraints.
More generally, for any monotone submodular objective,
we presented an approximation algorithm that exploits the
complementary nature of greedily selecting potential loop closures for
verification and
greedily broadcasting observations, in order to establish a performance guarantee. 
The performance guarantee in this more general case depends on resource budgets and the extent of perceptual ambiguity.

It remains an open problem whether constant-factor approximation 
for any monotone submodular objective is possible. We plan to study this open problem as
part of our future work.
Additionally, although the burden of verifying potential loop closures is
distributed among the robots, our current framework still relies on a
centralized scheme for evaluating the performance metric and running the
approximation algorithms presented for \ref{eq:codesign}.  In the future, we aim
to eliminate such reliance on centralized computation.

\subsubsection*{Acknowledgments} 
This work was supported in part by the NASA Convergent
Aeronautics Solutions project Design Environment for Novel
Vertical Lift Vehicles (DELIVER), by ONR under BRC award N000141712072, and by
ARL DCIST under Cooperative Agreement Number W911NF-17-2-0181.

% FIXME
%\bibliographystyle{splncs_srt}
%\vspace{-0.5cm}
{
  \footnotesize
\bibliographystyle{plainnat}
\bibliography{scanexchange,slam}
}

\clearpage
\begin{appendices}
	\renewcommand{\thesection}{\appendixname~\Alph{section}}
	
	\section{Algorithms}
	\label{app:algorithms}
	
	\begin{algorithm}[H]
		\caption{\textsc{\small Modular-Greedy} \small (\TotalUni{})}\label{alg:mgreedy}
		\begin{algorithmic}[1]
			\renewcommand{\algorithmicrequire}{\textbf{Input:}}
			\renewcommand{\algorithmicensure}{\textbf{Output:}}
			\Require
			\Statex - Exchange graph $\Gcal = (\Vall, \Eall)$
			\Statex - Communication budget $\kcomm$ and computation budget $\kcomp$ 
			\Statex - Modular $\fe: 2^{\Eall} \to \Rset_{\geq 0}$ and $\fv$ as defined in Section~\ref{sec:modular}
			\Ensure
			\Statex - A budget-feasible pair 
			$\Vgrd \subseteq \Vall, \Egrd \subseteq \Eall$. 
			%\vspace{0.08cm}
			\State {$\Vgrd \leftarrow \varnothing$}
	%		\State{$\Delta\leftarrow \min\,\{\kcomp,\kcomm\}$}
			\For{$i=1:\kcomm$} \Comment \textcolor{green!50!black}{greedy loop}
		     \State {$v^\star \leftarrow \argmax_{v \in \Vall\setminus\Vgrd}
				      				\fv(\Vgrd \cup \{v\})$}
		     \State {$\Vgrd \leftarrow \Vgrd \cup \{v^\star\}$}
		    \EndFor
		    \State $\Egrd \leftarrow \edg(\Vgrd)$
			\State \Return $\Vgrd, \Egrd$
		\end{algorithmic}
	\end{algorithm}
	
	\begin{algorithm}[H]
		\caption{\textsc{\small Edge-Greedy} \small (\TotalUni{})}\label{alg:egreedy}
		\begin{algorithmic}[1]
			\renewcommand{\algorithmicrequire}{\textbf{Input:}}
			\renewcommand{\algorithmicensure}{\textbf{Output:}}
			\Require
			\Statex - Exchange graph $\Gcal = (\Vall, \Eall)$
			\Statex - Communication budget $\kcomm$ and computation budget $\kcomp$ 
			\Statex - $\fe: 2^{\Eall} \to \Rset_{\geq 0}$
			\Ensure
			\Statex - A budget-feasible pair 
			$\Vgrd \subseteq \Vall, \Egrd \subseteq \Eall$. 
			%\vspace{0.08cm}
			\State {$\Egrd \leftarrow \varnothing$}
	%		\State{$\Delta\leftarrow \min\,\{\kcomp,\kcomm\}$}
			\For{$i=1:\min(\kcomm, \kcomp)$} \Comment \textcolor{green!50!black}{greedy loop} \label{alg:egrd_1_start}
		      \State {$e^\star \leftarrow \argmax_{e \in \Eall\setminus\Egrd}
		      				\fe(\Egrd \cup \{e\})$}
		      \State {$\Egrd \leftarrow \Egrd \cup \{e^\star\}$}
		    \EndFor 
		    \State $\Vgrd \leftarrow \textsc{\small VertexCover}(\Egrd)$ \Comment \textcolor{green!50!black}{any vertex cover of $\Egrd$} \label{alg:egrd_1_end}
		    
			\If {$\kcomp > \kcomm$} \Comment \textcolor{green!50!black}{local optimization} \label{alg:egrd_2_start}
			\State {$\Ecal_\text{free} \leftarrow 
					\edg(\Vgrd) \setminus \Egrd $} \Comment \textcolor{green!50!black}{identify comm-free edges}
			\For {$i=1:\min(|\Ecal_\text{free}|, \kcomp - \kcomm)$}
				\State {$e^\star \leftarrow \argmax_{e \in \Ecal_\text{free} \setminus \Egrd}
					      				\fe(\Egrd \cup \{e\})$}
				\State {$\Egrd \leftarrow \Egrd \cup \{e^\star\}$}
			\EndFor 
			\EndIf \label{alg:egrd_2_end}
			\State \Return $\Vgrd, \Egrd$
		\end{algorithmic}
	\end{algorithm}
	
	\begin{algorithm}[H]
		\caption{\textsc{\small Vertex-Greedy} \small (\TotalUni{})}\label{alg:vgreedy}
		\begin{algorithmic}[1]
			\renewcommand{\algorithmicrequire}{\textbf{Input:}}
			\renewcommand{\algorithmicensure}{\textbf{Output:}}
			\Require
			\Statex - Exchange graph $\Gcal = (\Vall, \Eall)$
			\Statex - Communication budget $\kcomm$ and computation budget $\kcomp$ 
			\Statex - $\fe: 2^{\Eall} \to \Rset_{\geq 0}$ and $\fv : \VV \mapsto
				    \fe\big(\edg(\VV)\big)$
			\Ensure
			\Statex - A budget-feasible pair 
			$\Vgrd \subseteq \Vall, \Egrd \subseteq \Eall$. 
			%\vspace{0.08cm}
			\State {$\Vgrd \leftarrow \varnothing$}
			\While {\textsc{\small True}} \Comment \textcolor{green!50!black}{greedy loop}
				\State {$v^\star \leftarrow \argmax_{v \in
						  \Vall\setminus\Vgrd} \fv(\Vgrd \cup \{v\})$} \Comment \textcolor{green!50!black}{find next best vertex}
				\If {$|\Vgrd \cup \{v^\star\}| > \kcomm$ \textbf{or} $|\edg(\Vgrd \cup \{v^\star\})| > \kcomp$ }
					\State \textbf{break} \Comment \textcolor{green!50!black}{stop if violating budget} \label{alg:vgrd_break}
				\EndIf
				\State {$\Vgrd \leftarrow \Vgrd \cup \{v^\star\}$}
			\EndWhile
			\State $\Egrd \leftarrow \edg(\Vgrd)$
			\State \Return $\Vgrd, \Egrd$
		\end{algorithmic}
	\end{algorithm}
	
%	\begin{algorithm}[H]
%		\caption{\textsc{\small Submodular-Greedy}}\label{alg:sgreedy}
%		\begin{algorithmic}[1]
%			\renewcommand{\algorithmicrequire}{\textbf{Input:}}
%			\renewcommand{\algorithmicensure}{\textbf{Output:}}
%			\Require
%			\Statex - Exchange graph $\Gcal = (\Vall, \Eall)$
%			\Statex - Communication budget $\kcomm$ and computation budget $\kcomp$ 
%			\Statex - $\fe: 2^{\Eall} \to \Rset_{\geq 0}$ and $\fv : \VV \mapsto
%				    \fe\big(\edg(\VV)\big)$
%			\Ensure
%			\Statex - A budget-feasible pair 
%			$\Vgrd \subseteq \Vall, \Egrd \subseteq \Eall$. 
%			\vspace{0.08cm}
%			\State $\Vgrd^{(e)}, \Egrd^{(e)} \leftarrow  \textsc{\small Edge-Greedy}(\Gcal, \kcomm, \kcomp, \fe)$ 
%			\State $\Vgrd^{(v)}, \Egrd^{(v)} \leftarrow  \textsc{\small Vertex-Greedy}(\Gcal, \kcomm, \kcomp, \fv)$ 
%			\If{$\fe(\Egrd^{(e)}) > \fe(\Egrd^{(v)})$}
%				\State $\Vgrd \leftarrow \Vgrd^{(e)}, \Egrd \leftarrow \Egrd^{(e)}$
%			\Else
%				\State $\Vgrd \leftarrow \Vgrd^{(v)}, \Egrd \leftarrow \Egrd^{(v)}$
%			\EndIf
%			\State \Return $\Vgrd, \Egrd$
%		\end{algorithmic}
%	\end{algorithm}
	
	\section{Proofs}
	\label{app:proofs}
	\begin{proof}[Theorem~\ref{th:fvNMS}]
	  For convenience, we present the proof after modifying the definition of
	  $\fv$ according to
	  \begin{align}
		\fv(\Vcal) = \underset{\EE
		  \subseteq \edg(\Vcal) \cup \mathcal{N}}{\max}
		  \hspace{.1cm} \sum_{e \in \Ecal} p(e) \quad \text{s.t.} \quad |\Ecal|
		  = k
		\end{align}
		where now $\Ncal$ is a set of $k$ \emph{null} items with $p(e) = 0$ for
		all $e \in \Ncal$. It is easy to see this modification does not change
		the function $\fv$.
		\begin{itemize}
			\item[$\diamond$] Normalized: $\fv(\varnothing) = 0$ by definition.
			\item[$\diamond$] Monotone: for any $\Scal \subseteq \Qcal \subseteq \Vall$,
			      $\edg(\Scal) \subseteq \edg(\Qcal)$ and thus $\fv(\Scal) \leq \fv(\Qcal)$.
			\item[$\diamond$] Submodular: we need to show that for any $\Scal
				      \subseteq \Qcal \subseteq \Vall$ and all $v \in \Vall/\Qcal$:
			      \begin{equation}
				      \fv(\Scal \cup \{v\}) - \fv(\Scal) \geq \fv(\Qcal \cup \{v\}) - \fv(\Qcal)
				      \label{eq:sub_def}
			      \end{equation}
				  If $\fv(\Qcal \cup \{v\}) = \fv(\Qcal)$, \eqref{eq:sub_def}
				  follows from the monotonicity of $g$. We thus focus on cases
				  where $\fv(\Qcal \cup \{v\}) > \fv(\Qcal)$. Recall that for
				  any $\Vcal \subseteq \Vall$,
				  $\fv(\Vcal)$ is the sum of top $k$ edge probabilities in
				  $\edg(\Vcal) \cup \Ncal$. Let $\edg(\Vcal;k)$ represent such a set. 
%				  For
%				  convenience, we
%				  can always assume $|\edg(\Vcal;k)| = k$ by adding a sufficient
%				  number of null items with zero
%				  probability.\footnote{This is equivalent to modifying the
%					definition of $\fv$ according to
%					\begin{align}
%					  \fv(\Vcal) = \underset{\EE
%						\subseteq \edg(\Vcal) \cup \mathcal{N},|\EE| = k}{\max}
%						\hspace{.1cm} \sum_{e \in \Ecal} p(e)
%					\end{align}
%					where $\Ncal$ is a set of $k$ null items with $p(e) = 0$ for
%					all $e \in \Ncal$.
%				  }
				  We have,
				  \begin{align}
					\fv(\Qcal \cup \{v\}) - \fv(\Qcal) & = 
%					\sum_{e \in
%					  \edg(\Qcal \cup \{v\};k)} p(e) \\
%					  &= 
					  \sum_{e \in
					\EE_\boxplus} p(e) - \sum_{e \in \EE_\boxminus}
					  p(e)
					  \label{eq:Q}
				  \end{align}
				  where $\EE_\boxplus \triangleq \edg(\Qcal \cup \{v\};k) \setminus
				  \edg(\Qcal;k)$ and $\EE_\boxminus \triangleq \edg(\Qcal;k) \setminus
				  \edg(\Qcal \cup \{v\};k)$. We know that 
				  $|\EE_\boxplus| = |\EE_\boxminus|$ since $|\edg(\Qcal \cup
				  \{v\};k)| = |\edg(\Qcal;k)| = k$. Now we claim that,
				  \begin{align}
					\fv(\Scal \cup \{v\}) \geq \fv(\Scal) + 
					   \sum_{e \in
					\EE_\boxplus} p(e) - \sum_{e \in \EE_\boxminus^\ast}
					  p(e)
					  \label{eq:ineq}
				  \end{align}
				  in which $\EE_\boxminus^\ast$ is the set of $|\EE_\boxplus|$
				  edges in $\edg(\Scal;k)$ with lowest probabilities. To see this, note that 
				  $\edg(\Scal;k) \cup \EE_\boxplus \setminus
				  \EE^\ast_\boxminus$ is a $k$-subset of $\edg(\Scal \cup \{v\})
				  \cup \Ncal$. Therefore,
				  \begin{align}
					\fv(\Scal \cup \{v\}) & \triangleq \underset{\EE
		  \subseteq \edg(\Vcal) \cup \mathcal{N}}{\max}
		  \hspace{.1cm} \sum_{e \in \Ecal} p(e) \quad \text{s.t.} \quad |\Ecal|
		  = k\\ 
		  &\geq
		  \sum_{e \in \edg(\Scal;k) \cup \EE_\boxplus \setminus
		  \EE^\ast_\boxminus} p(e) \\
		  & = \fv(\Scal) + 
					   \sum_{e \in
					\EE_\boxplus} p(e) - \sum_{e \in \EE_\boxminus^\ast}
					  p(e)
					\label{}
				  \end{align}
%				  
%				  
%				  consequently, this set is a feasible choice for the optimization problem
%				  that appears in
%				  the definition of $g(\Scal \cup \{v\})$. Now \eqref{eq:ineq} trivially holds because the
%				  RHS is the objective value associated to this feasible
%				  solution. Using \eqref{eq:Q} and \eqref{eq:ineq}, we only need
%				  to show that $\sum_{e \in \EE_\boxminus}
%					p(e) \geq \sum_{e \in \EE_\boxminus^\ast} p(e)$. 
				    %From \eqref{eq:Q} and \eqref{eq:ineq}, 
%					It only remains to show that $\sum_{e \in \EE_\boxminus}
%					p(e) \geq \sum_{e \in \EE_\boxminus^\ast} p(e)$.
					Now let us
					sort edges according to their probabilities in
					$\edg(\Scal;k)$ and in $\edg(\Qcal;k)$. Since
					$\Scal \subseteq \Qcal$ and, consequently, $\edg(\Scal)
					\subseteq \edg(\Qcal)$, the probability of the
					$i$th most-probable 
					edge in 
					$\edg(\Qcal;k)$ is at least that of the $i$th most-probable
					edge in $\edg(\Scal;k)$. This shows that 
					$\sum_{e \in \EE_\boxminus}
					p(e) \geq \sum_{e \in \EE_\boxminus^\ast} p(e)$, which
					together with \eqref{eq:Q} and \eqref{eq:ineq}  
					conclude the proof.

		\end{itemize}
	\end{proof}
	
	%\vspace{0.5cm}
	\begin{lemma}
	\normalfont
	For any $\Vcal \subseteq \Vall$ that is \ref{prob:codesign_nested}-feasible, define
	\begin{equation}
	\edg(\VV; \kcomp) \in 
	\underset{\Ecal \subseteq
	\edg(\Vcal)
	}{\argmax} 
	\fe(\Ecal) \text{ s.t. } |\Ecal| \leq \kcomp.
	\label{eq:edgek}
	\end{equation}
	Then, $\edg(\VV;\kcomp)$ is \ref{eq:codesign}-feasible.
	\label{lem:ecal_feasible}
	\end{lemma}
	\begin{proof}[Lemma~\ref{lem:ecal_feasible}]
	By construction, $|\edg(\VV;\kcomp)| \leq \kcomp$. Also, there exists a cover of $\edg(\VV;\kcomp)$ (namely, $\VV$) that satisfies \textbf{CB}.
	This concludes the proof. 
	\end{proof}

	%\vspace{0.5cm}
	\begin{lemma}
	\normalfont
	Let $\OPTe$ and $\OPTv$ denote the optimal values of \ref{eq:codesign} and \ref{prob:codesign_nested}, respectively. Then
	$\OPTe = \OPTv$.
	\label{lem:opt_eq}
	\end{lemma}
	\begin{proof}[Lemma~\ref{lem:opt_eq}]
	We prove $\OPTe \geq \OPTv$ and $\OPTv \geq \OPTe$.
	\begin{itemize}
	\item[$\diamond$] $\OPTe \geq \OPTv$: Suppose $\Vcal^\star$ is an optimal solution to
			\ref{prob:codesign_nested}. Let $\Ecal^\star \triangleq \edg(\VV^\star;\kcomp)$.
			By Lemma~\ref{lem:ecal_feasible}, $\Ecal^\star$ is
			\ref{eq:codesign}-feasible. In addition, $\fe(\Ecal^\star) = \OPTv$
			by the definition of $\fv$. Therefore, $\OPTe \geq \fe(\Ecal^\star) = \OPTv$.
	\item[$\diamond$] $\OPTv \geq \OPTe$: Suppose $\Ecal^\star$ is an optimal solution to
		\ref{eq:codesign}. Then there exists $\Vcal^\star \in \Cover(\Ecal^\star)$ that satisfies \textbf{CB}.
		Thus $\Vcal^\star$ is 
		\ref{prob:codesign_nested}-feasible. Also,
		$\fv(\Vcal^\star) \geq \OPTe$ by the definition of $\fv$.
		Therefore, $\OPTv \geq \fv(\Vcal^\star) \geq \OPTe$.
	\end{itemize}
	\end{proof}	
	
	%\vspace{0.5cm}
	\begin{proof}[Theorem~\ref{thm:modular_apx}]
	Let $\tilde{\Ecal}$ be the edge set returned by the procedure. 
	By Lemma~\ref{lem:ecal_feasible}, $\tilde{\Ecal}$ is \ref{eq:codesign}-feasible. Now,
	\begin{align}
		\fe(\tilde{\Ecal}) &= \fv(\VV) &&\text{(def. of $\fv$)}\\
		& \geq \alpha \cdot \OPTv && \text{(assumption on \textsf{ALG})}\\
		& = \alpha \cdot \OPTe. && \text{(Lemma~\ref{lem:opt_eq})}
    \end{align}
	\end{proof}
	
	%\vspace{0.5cm}
	\begin{lemma}
	\normalfont
	\textsc{\small Edge-Greedy} (Algorithm~\ref{alg:egreedy}) is an $\alpha_e(\kcomm, \kcomp)$-approximation algorithm for \ref{eq:codesign} under \TotalUni{}, where
	\begin{align}
			\alpha_e(\kcomm, \kcomp) \triangleq 1-\exp \big (-\min\,\{1,\kcomm/\kcomp\} \big ).
			\label{eq:egrd_apx1}
	\end{align}
	\label{lem:egrd_apx}
	\end{lemma}
	\begin{proof}[Lemma~\ref{lem:egrd_apx}]
	
	We first show that Algorithm~\ref{alg:egreedy} returns a feasible solution to \ref{eq:codesign}. The algorithm proceeds in two phases. 
	In phase I (line~\ref{alg:egrd_1_start}-\ref{alg:egrd_1_end}),
	we greedily select $\min(\kcomm, \kcomp)$ edges. This clearly satisfies the computation budget $\kcomp$.
	Then, we find an arbitrary vertex cover of the selected edges (line~\ref{alg:egrd_1_end}),
	e.g., by selecting one of the two vertices for each edge. For the rest of the algorithm, we fix this vertex cover 
	as the subset of vertices we return. 
	Note that in the worst case, the selected edges will be disjoint.
	Consequently, the size of the vertex cover is at most $\min(\kcomm, \kcomp)$, satisfying the communication budget $\kcomm$. 
	Thus, by the end of phase I, the constructed solution is \ref{eq:codesign}-feasible. 
	
	Phase II (line~\ref{alg:egrd_2_start}-\ref{alg:egrd_2_end}) improves the current solution while ensuring feasibility. 
	Given the selected vertices, we identify the set of unselected edges that are ``communication-free'', 
	i.e., no extra communication cost will be incurred if we select any of these edges. 
	In other words, an edge is ``communication-free'' if it is already covered by the current vertex cover. 
	In Phase II, the algorithm simply adds ``communication-free'' edges greedily, until there is no computation budget left. The final solution
	thus remains \ref{eq:codesign}-feasible. 
	
	Next, we prove the performance guarantee presented in the lemma. To do so, we only need to look at 
	the solution after Phase I. 
	Let $\OPTe$ denote the optimal value of \ref{eq:codesign}.
	Consider the relaxed version of \ref{eq:codesign} where we remove the communication budget,
	\begin{equation}
		\underset{\Ecal \subseteq \Eall}{\text{maximize }} \fe(\Ecal) \text{ s.t. } |\Ecal| \leq \kcomp.
		\label{prob:ijrr_pe}
	\end{equation}
	Let OPT$_e$ denote the optimal value of the relaxed problem. Clearly, OPT$_e$ $\geq \OPTe$.
	Let $\Egrd$ be the set of selected edges after Phase I. By construction, $|\Egrd| \geq \min(\kcomm, \kcomp)$. 
	Thus,
	\begin{align}
		\fe(\Egrd) 
		& \geq \Big ( 1-\exp \big (- \min\{\kcomm, \kcomp\}/\kcomp \big ) \Big) \cdot \text{\text{OPT}$_e$} && \text{(\cite[Theorem~$1.5$]{krauseSurvey})} \\
		& = \alpha_e(\kcomm, \kcomp) \cdot  \text{\text{OPT}$_e$}  && \text{(def. of $\alpha_e$)}\\
		& \geq \alpha_e(\kcomm, \kcomp) \cdot \OPTe. && \text{(\text{OPT}$_e \geq \OPTe$)} 
	\end{align}
	Finally, note that because of Phase II, $\min(\kcomm, \kcomp)$ is only a lower bound on the number of edges selected by \textsc{\small E-Greedy}.
%	Consequently, given a specific instance of $\Gcal$, the \emph{empirical} approximation ratio achieved by the algorithm can be higher than $\alpha_e(\kcomm, \kcomp)$.
	\end{proof}

	%\vspace{0.5cm}
	\begin{lemma}
	\normalfont
	Let $\Delta$ be the maximum vertex degree in $\Gcal$. \textsc{\small Vertex-Greedy} (Algorithm~\ref{alg:vgreedy}) is an $\alpha_v(\kcomm, \kcomp, \Delta)$-approximation algorithm for \ref{eq:codesign} under \TotalUni{}, where
	\begin{align}
		\alpha_v(\kcomm, \kcomp, \Delta) & \triangleq 1-\exp \big (-\min\,\{1, \lfloor \kcomp / \Delta \rfloor/\kcomm \} \big).
		\label{eq:vgrd_apx1}
	\end{align}
	\label{lem:vgrd_apx}
	\end{lemma}
	\begin{proof}[Lemma~\ref{lem:vgrd_apx}]
	
	As shown in Algorithm~\ref{alg:vgreedy}, we terminate the greedy loop if the next selected vertex violates 
	either one of the budgets (line~\ref{alg:vgrd_break}). Thus, the returned solution is guaranteed to be \ref{eq:codesign}-feasible. 
	Note that the number of vertices we select is at least $\min(\kcomm, \lfloor \kcomp / \Delta \rfloor)$.
	This is because whenever we select less than this number of vertices, there is guaranteed to be enough computation and communication budgets to select the next vertex and include all its incident edges.
	
	Next, we prove the performance guarantee presented in the lemma. Let $\OPTe$ denote the optimal value of \ref{eq:codesign}.
	Consider the relaxed version of \ref{eq:codesign} under \TotalUni{} where we remove the computation budget,
	\begin{equation}
		\underset{\Ecal \subseteq \Eall}{\text{maximize }} \fe(\Ecal) \text{ s.t. } \exists \, \Vcal \in \Cover(\EE) \text{ s.t. } |\Vcal| \leq \kcomm.
		\label{prob:rss_pe}
	\end{equation}
	By \cite[Theorem~$2$]{tian18}, there is an approximation-preserving reduction from the above problem to the following problem,
	\begin{equation}
		\underset{\Vcal \subseteq \Vall}{\text{maximize }} h(\Vcal) \text{ s.t. } |\Vcal| \leq \kcomm.
		\label{prob:rss_pv}
	\end{equation}
	where $h: 2^{\Vall} \to \Rset_{\geq 0}: \Vcal \mapsto \fe(\edg(\Vcal))$.
	Let OPT$_v$ and OPT$^\prime_v$ denote the optimal values of 
	(\ref{prob:rss_pe}) and (\ref{prob:rss_pv}), respectively. By \cite[Lemma~$1.2$]{tian18}, OPT$_v$ $=$ OPT$^\prime_v$.
	Clearly, OPT$_v$ $\geq \OPTe$.
	Let $\Egrd, \Vgrd$ be the edges and vertices selected by \textsc{\small Vertex-Greedy}.
	By construction, $|\Vgrd| \geq \min(\kcomm, \lfloor \kcomp / \Delta \rfloor)$. 
	Thus,
	\begin{align}
		\fe(\Egrd) & = h(\Vgrd)  && \text{(def. of $h$)}\\
		& \geq \Big ( 1-\exp \big (- \frac{\min\{\kcomm, \lfloor \kcomp / \Delta \rfloor\}}{\kcomm} \big ) \Big) \cdot \text{\text{OPT}$^\prime_v$} && \text{(\cite[Theorem~$1.5$]{krauseSurvey})}\\
		& = \alpha_v(\kcomm, \kcomp, \Delta) \cdot \text{\text{OPT}$^\prime_v$} && \text{(def. of $\alpha_v$)} \\
		& = \alpha_v(\kcomm, \kcomp, \Delta) \cdot \text{\text{OPT}$_v$} && \text{(\cite[Lemma~$1.2$]{tian18})} \\
		& \geq \alpha_v(\kcomm, \kcomp, \Delta) \cdot \OPTe. && \text{(\text{OPT}$_v \geq \OPTe$)} 
	\end{align}
	As mentioned above, $\min(\kcomm, \lfloor \kcomp / \Delta \rfloor)$ is only a lower bound on the number of vertices selected by \textsc{\small V-Greedy}.
	\end{proof}
	
	%\vspace{1cm}
	\begin{proof}[Theorem~\ref{thm:submodular_apx}] By Lemma~\ref{lem:egrd_apx} and Lemma~\ref{lem:vgrd_apx},
	\begin{align}
		\alpha(\kcomm, \kcomp, \Delta) &= \max \, \{\alpha_e(\kcomm, \kcomp), \alpha_v(\kcomm, \kcomp, \Delta) \} \\
		& = 1 - \exp \big (-\min \big \{ 1, \max \, ( \kcomm / \kcomp, \lfloor \kcomp / \Delta \rfloor / \kcomm) \big \} \big) \\
		& = 1 - \exp \big (-\min\{1, \gamma\} \big ).
	\end{align}
	\end{proof}
	
	%%%%%%%%%%%%%%%%%%%%%%%%%%%%%%%%%%%%%%%%%%%%%%%%%%%%%%%%%%%%%%%%%%%%%%%%%%
	%%%%%%%%%%%%%%%%%%%%%%%%%%%%%%%%%%%%%%%%%%%%%%%%%%%%%%%%%%%%%%%%%%%%%%%%%%
	%%%%%%%%%%%%%%%%%%%%%%%%%%%%%%%%%%%%%%%%%%%%%%%%%%%%%%%%%%%%%%%%%%%%%%%%%%
	\section{Additional Results}
	\label{app:additional_results}
	
	\begin{figure}[H]
			\centering
			\hfill
			\begin{subfigure}[t]{0.30\textwidth}
				\centering
				\includegraphics[width=\textwidth]{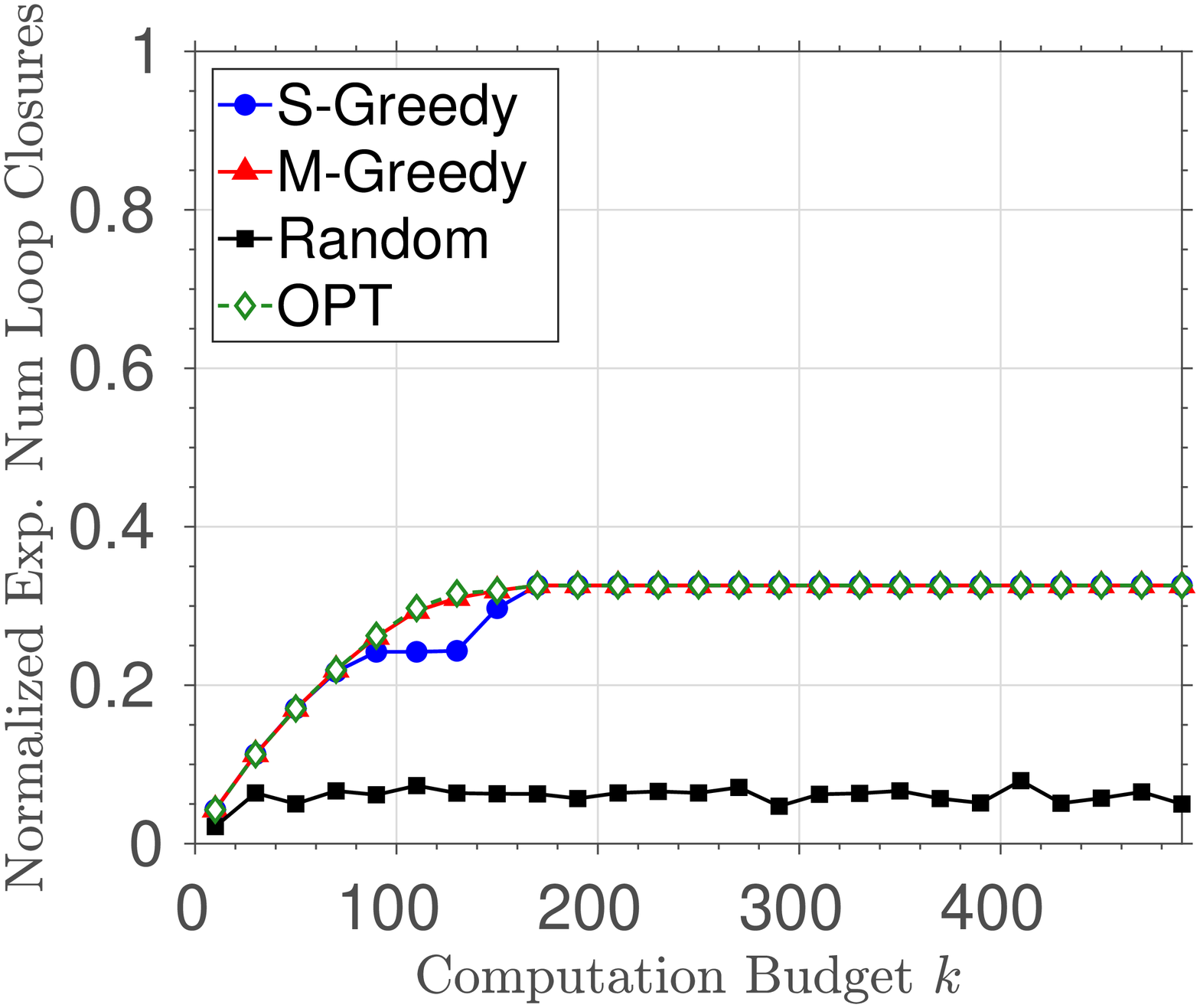}
				\caption{\small $b = 20$}
				\label{fig:kitti_maxprob_b_20}
			\end{subfigure}
			\hfill
			\begin{subfigure}[t]{0.30\textwidth}
				\centering
				\includegraphics[width=\textwidth]{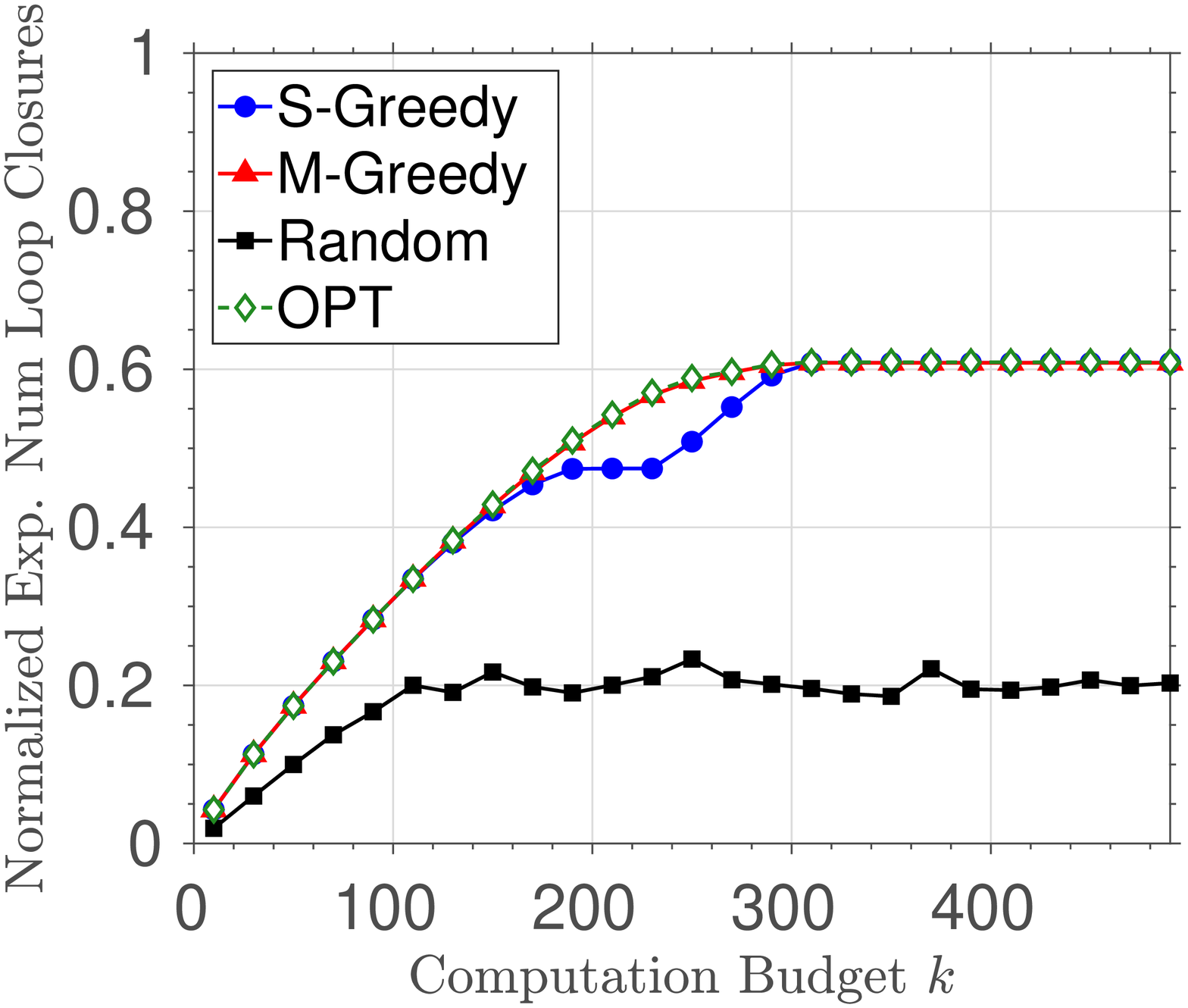}
				\caption{\small $b = 70$}
				\label{fig:kitti_maxprob_b_70}
			\end{subfigure}
			\hfill
			\begin{subfigure}[t]{0.30\textwidth}
				\centering
				\includegraphics[width=\textwidth]{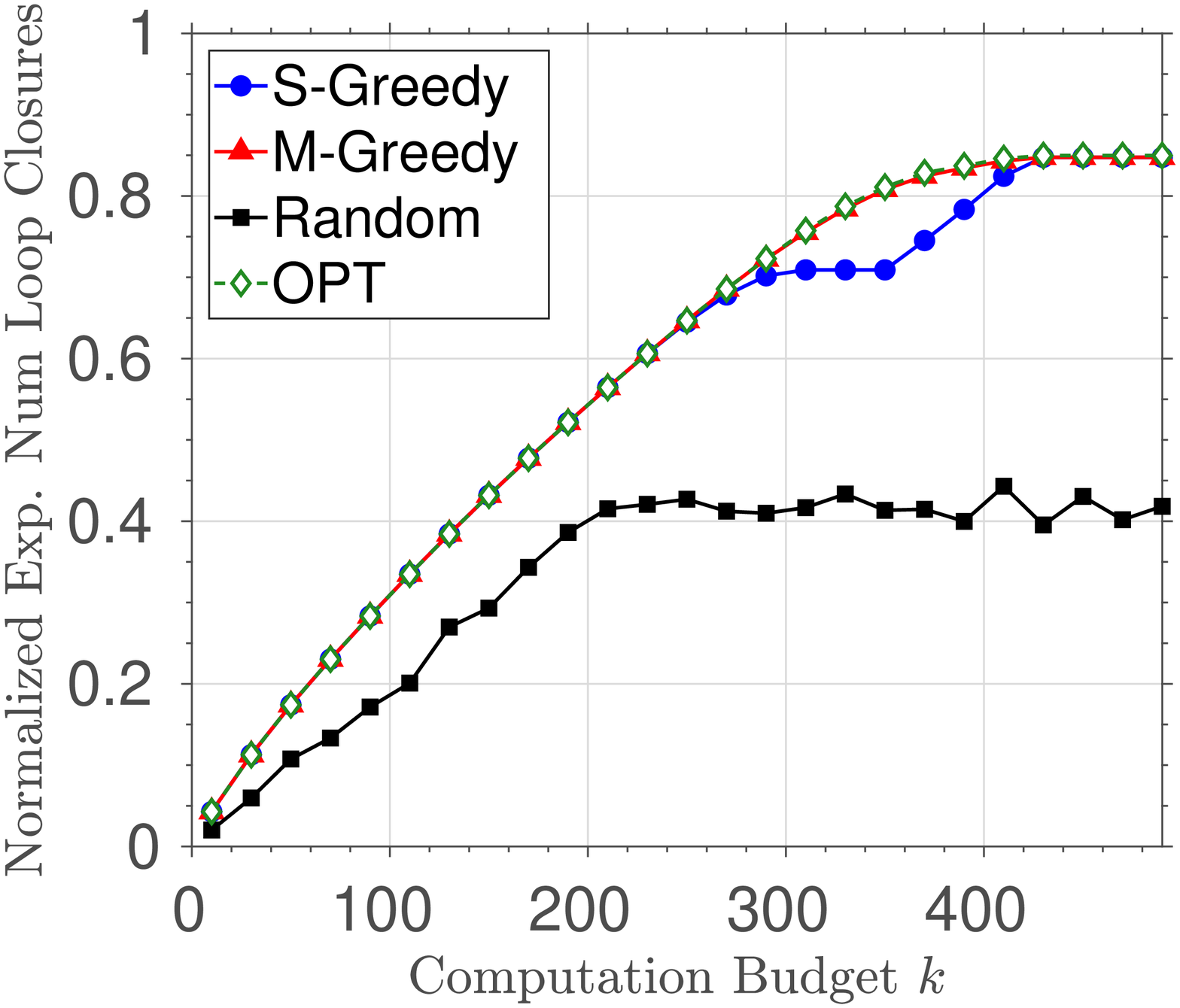}
				\caption{\small $b = 150$}
				\label{fig:kitti_maxprob_b_150}
			\end{subfigure}
			\hspace*{\fill}%
			\caption{\small Comparison of $\textsc{\small M-Greedy}$ and $\textsc{\small S-Greedy}$ when maximizing the expected number of true loop closures in KITTI~00.
			Each subplot shows one scenario with a fixed communication budget $\kcomm$ and varying computation budget $\kcomp$. For completeness we also include the random baseline and the optimal values (OPT) computed by solving the corresponding ILP.
			}
			\label{fig:kitti_maxprob_fix_b}
		\end{figure}
	
	\begin{figure}[H]
%		\centering
%		\hfill
%		\begin{subfigure}[t]{0.30\textwidth}
%			\centering
%			\includegraphics[width=\textwidth]{figures/atlas_fim_b_20.eps}
%			\caption{\small $b = 20$}
%			\label{fig:atlas_fim_b_20}
%		\end{subfigure}
%		\hfill
%		\begin{subfigure}[t]{0.30\textwidth}
%			\centering
%			\includegraphics[width=\textwidth]{figures/atlas_fim_b_40.eps}
%			\caption{\small $b = 40$}
%			\label{fig:atlas_fim_b_70}
%		\end{subfigure}
%		\hfill
%		\begin{subfigure}[t]{0.30\textwidth}
%			\centering
%			\includegraphics[width=\textwidth]{figures/atlas_fim_b_70.eps}
%			\caption{\small $b = 70$}
%			\label{fig:atlas_fim_b_150}
%		\end{subfigure}
%		\hspace*{\fill}%
%		\\ 
		\centering
		\hspace*{\fill}%
		\begin{subfigure}[t]{0.30\textwidth}
			\centering
			\includegraphics[width=\textwidth]{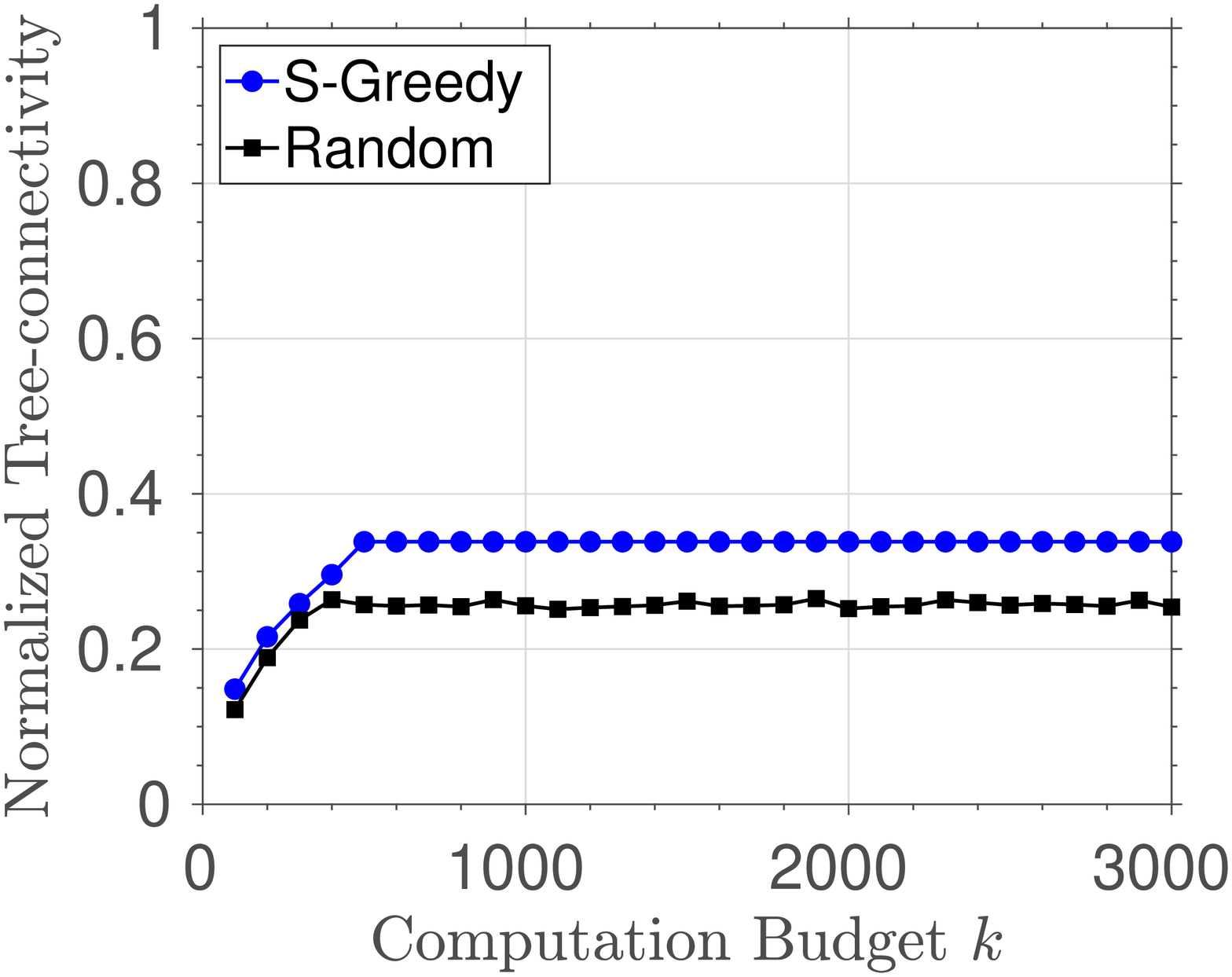}
			\caption{\small $b = 50$}
			\label{fig:atlas_wst_b_50}
		\end{subfigure}
		\hfill
		\begin{subfigure}[t]{0.30\textwidth}
			\centering
			\includegraphics[width=\textwidth]{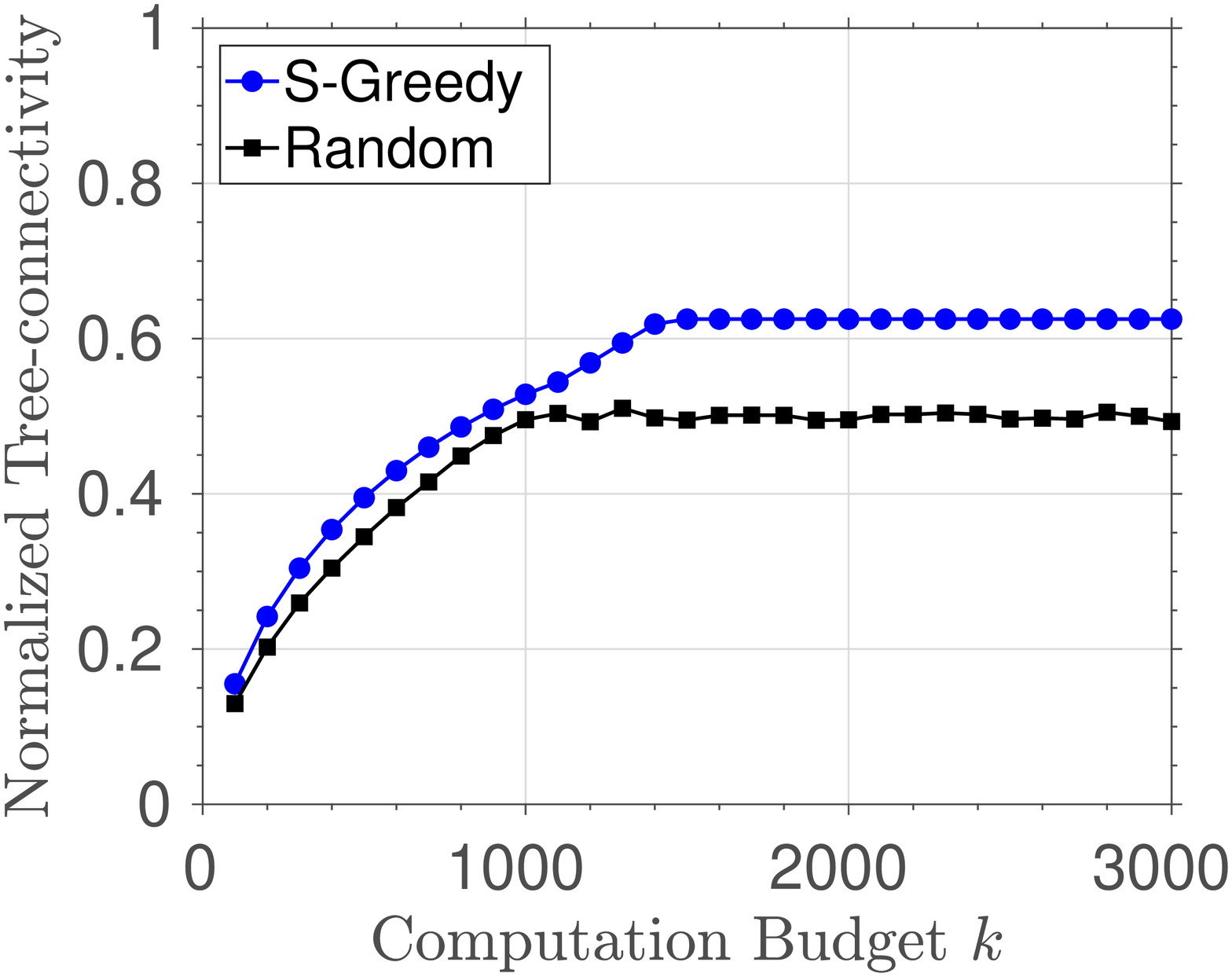}
			\caption{\small $b =150$}
			\label{fig:atlas_wst_b_150}
		\end{subfigure}
		\hfill
		\begin{subfigure}[t]{0.30\textwidth}
			\centering
			\includegraphics[width=\textwidth]{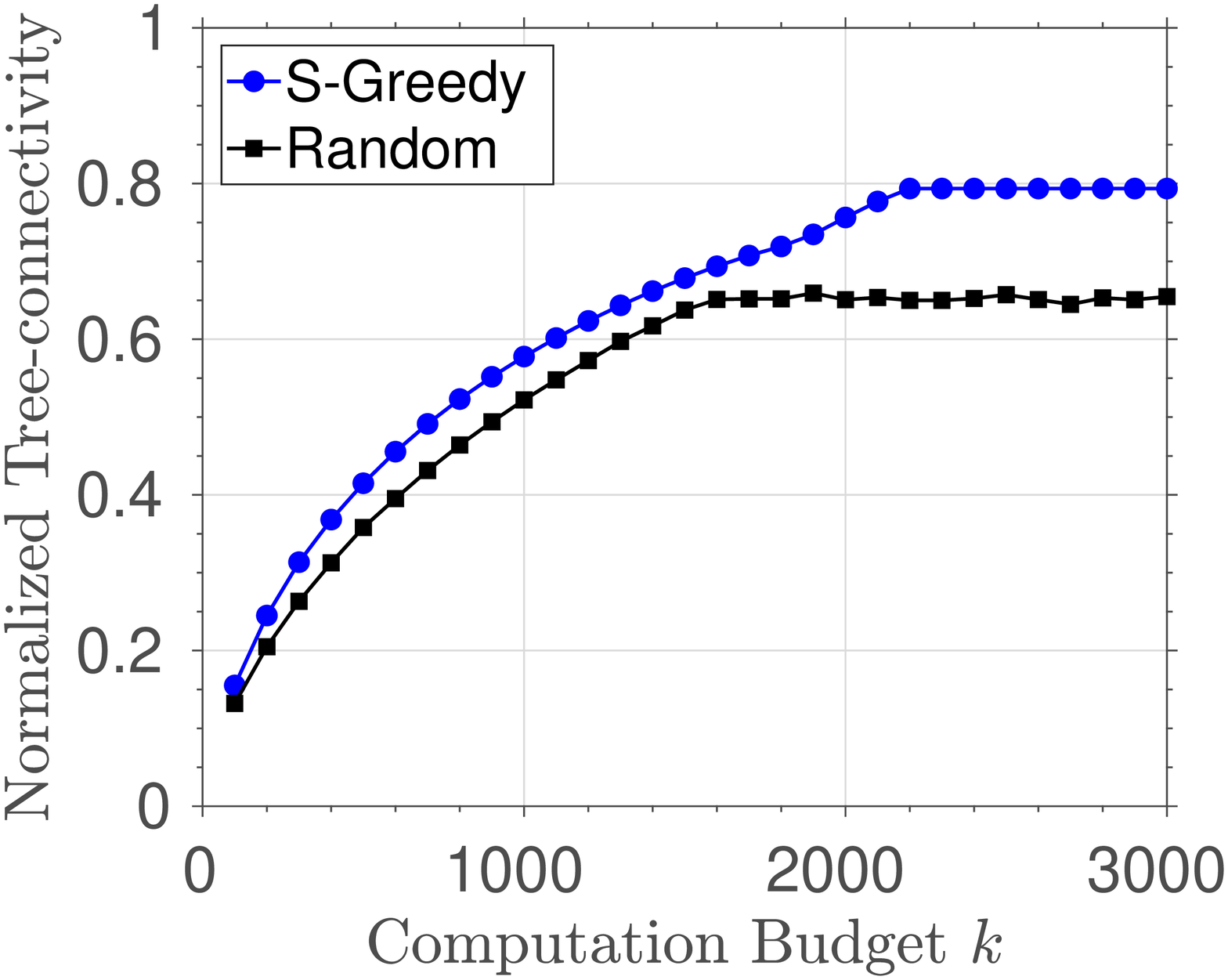}
			\caption{\small $b = 250$}
			\label{fig:atlas_wst_b_250}
		\end{subfigure}
		\hspace*{\fill}%
		\caption{\small Performance of \textsc{\small Submodular-Greedy} in simulation
			under the \TotalUni{} regime,
			with the tree-connectivity objective.
			Each subplot shows one scenario with a fixed communication budget $\kcomm$ and varying computation budget $\kcomp$. 
			The proposed algorithm is compared against a random baseline. In this case, we do not compute the upper bound (UPT) using convex relaxation
			as it is too time-consuming.
		}
		\label{fig:atlas_wst_fix_b}
	\end{figure}

\end{appendices}

\end{document}